\documentclass[11pt]{article}
\usepackage[paper=a4paper,layout=letterpaper,hmargin={1in,1in},vmargin=1in]{geometry}
\usepackage[usenames,dvipsnames]{xcolor}
\usepackage{times}
\usepackage[italic]{mathastext}

\usepackage{amssymb}
\usepackage{amsthm}
\usepackage{amsmath}
\usepackage{listings}
\usepackage[ruled,lined,linesnumbered]{algorithm2e} 
\usepackage{subfig} 
\usepackage{xcolor}
\usepackage{listings}
\usepackage{lstlangAMPL}
\usepackage{tkz-graph}
\usetikzlibrary{calc}
\usetikzlibrary{decorations.markings}
\usepackage{array}
\usepackage{multirow}
\usepackage{graphicx}
\newcommand{\cv}{\mathbf{c}}
\newcommand{\xv}{\mathbf{x}}
\newcommand{\yv}{\mathbf{y}}
\newcommand{\hv}{\mathbf{h}}
\newcommand{\bv}{\mathbf{b}}
\newcommand{\wv}{\mathbf{w}}

\DeclareMathOperator{\hb}{hb}

\DeclareMathOperator{\body}{body}   
\DeclareMathOperator{\head}{head} 
\DeclareMathOperator{\Var}{Var}
\DeclareMathOperator{\lhm}{LH}

\DeclareMathOperator{\imp}{:-}

\definecolor{mygray}{rgb}{0.4,0.4,0.4}
\definecolor{mygreen}{rgb}{0,0.8,0.6}
\definecolor{myorange}{rgb}{1.0,0.4,0}
\definecolor{darkgreen}{rgb}{0,0.6,0.0}

\newcommand{\+}[1]{\ensuremath{\mathbf{#1}}}
\newtheorem{theorem}{Theorem}%[section]

\newtheorem{proposition}[theorem]{Proposition}

\newtheorem{definition}[theorem]{Definition}

\newenvironment{todo}{\color{red} Todo: }{}

\newcommand{\argmin}{\operatornamewithlimits{argmin}}

\DeclareMathOperator{\nb}{nb}
\newcommand{\colorv}{\mbox{\em color}} 
\newcommand{\aut}{\mbox{\em Aut}}

\newcommand{\lp}{{\cal LP}}

\newcommand{\needcite}[1]{}

\newcommand{\be}{\begin{equation}}
\newcommand{\ee}{\end{equation}}
\newcommand{\nbe}{\begin{equation*}}
\newcommand{\nee}{\end{equation*}}
\newcommand{\bea}{\begin{eqnarray*}}
\newcommand{\eea}{\end{eqnarray*}}

\newcommand{\ignore}[1]{}

\renewcommand{\eqref}[1]{Eq.~\ref{#1}}

\renewcommand{\iff}{\Leftrightarrow} 
\renewcommand{\implies}{\Rightarrow}

\newenvironment{myalign*}
 {\footnotesize\csname flalign*\endcsname}
 {\csname endflalign*\endcsname\ignorespacesafterend}

\begin{document}


\title{Relational Linear Programs}

\author{
  Kristian Kersting\thanks{Martin Mladenov and Kristian Kersting were
    supported by the German Research Foundation DFG, KE 1686/2-1,
    within the SPP 1527, and the German-Israeli Foundation for
    Scientific Research and Development, 1180-218.6/2011.}\\
  \normalsize
    TU Dortmund University\\
  \and 
  Martin Mladenov$^*$\\ 
  \normalsize
    TU Dortmund University\\
    \and
  Pavel Tokmakov\\
  \normalsize
    University of Bonn, Germany\\
  }
\date{}
\maketitle

\begin{abstract}
We propose relational linear programming, a simple framework for combing linear programs (LPs) and logic programs. 
A relational linear program (RLP) is a declarative LP template defining
the objective and the constraints through the logical concepts of objects, relations, and quantified variables. This allows one to 
express the LP objective and constraints relationally for a varying number of individuals and relations among them without enumerating them.
Together with a logical knowledge base, effectively a logical program consisting of logical facts and rules, it induces a ground LP. This ground 
LP is solved using lifted linear programming. That is, symmetries within the ground LP are employed to reduce its dimensionality, if possible, and 
the reduced program is solved using any off-the-shelf LP solver. In contrast to mainstream LP template languages like AMPL,  which features a 
mixture of declarative and imperative programming styles, RLP's relational nature allows a more intuitive representation of optimization problems 
over relational domains. We illustrate this empirically by experiments on approximate inference in Markov logic networks using LP relaxations, 
on solving Markov decision processes, and on collective inference using LP support vector machines.
\end{abstract}


%!TEX root = main.tex
\section{Introduction}
\label{intro}
Modern social and technological trends result in an enormous increase in the amount of accessible data, with a significant portion of the resources being interrelated in a complex way and having inherent uncertainty. Such data, to which we may refer to as relational data, arise
for instance in social network and media mining, natural language processing, open information extraction, the web, bioinformatics, and 
robotics, among others, and
typically features substantial social and/or business value if become amenable to computing machinery. 
Therefore it is not surprising
that probabilistic logical languages, see e.g.~\cite{Getoor:2007aa,deraedt08springer,DeRaedt:2008aa,deraedt10srl} and references in
there for overviews, are currently provoking a lot of new AI research with tremendous theoretical and 
practical implications. By combining aspects of logic and probabilities --- a dream of AI dating back to at least the late 1980s when  Nils Nilsson introduced
the term probabilistic logics~\cite{nilsson86} --- they 
help to effectively manage both complex interactions and uncertainty in the data.
Moreover, since performing inference using traditional approaches within these languages is in principle rather costly, as tractability of traditional
inference approaches comes at the price of either coarse approximations (often without any guarantees) or restrictions in the 
language, they have motivated novel forms of inference. In essence, 
probabilistic logical models can be viewed as collections of 
building blocks respectively templates such as weighted clauses that are instantiated several times to construct a ground probabilistic model. 
If few templates are instantiated often, the resulting
ground model is likely to exhibit symmetries. In his seminal paper~\cite{poole03}, David Poole suggested to exploited these symmetries to speed
up inference within probabilistic logic models. This has motivated an active field of research known as 
lifted probabilistic inference, see e.g.~\cite{kersting12faia} and references in there.

However, instead of looking at AI through the glasses of probabilities over possible worlds, we may also approach it using
optimization. That is, we  have a preference relation over possible worlds, and
we want a best possible world according to the preference. The preference is often to
minimize some objective function.  Consider for example a typical machine learning user in 
%Actually, many machine learning and AI problems reduce
%to optimization problems. 
action solving a problem for
some data. She selects
%The modeler formulates the problem by selecting an appropriate family of
%mathematical models and massages the data amenable to modeling. 
%In contrast, many machine learning algorithms tend to be tractable by design by virtue of the powerful framework of (convex) mathematical programming. In fact, a typical machine learning application consists of 
%The modeler selects 
a model for the underlying phenomenon to be learned (choosing a learning bias), formats the raw data according to the chosen model, and then tunes the model parameters by minimizing some objective function induced by the data and the model assumptions. In the process
of model selection and validation, the core optimization problem in the last step may be solved many times.
Ideally, the optimization problem solved in the last step falls within a class of mathematical programs for which efficient and robust solvers are available. For example, linear, semidefinite and quadratic programs, found at the heart of many popular AI and 
learning algorithms, can be solved efficiently by commercial-grade software packages. 

This is an instance of the declarative
``$\text{Model} + \text{Solver}$''
paradigm currently observed a lot in AI~\cite{geffner14},  machine learning and also 
data mining~\cite{gunsNR11}:  instead of outlining how a solution should
be computed, we specify what the problem is using some high-level modeling language and solve it using general solvers.

Unfortunately, however, today's solvers for mathematical programs typically require that the mathematical program is presented in some canonical algebraic form or offer only some very restricted modeling environment. For example, a solver may require that a set of linear constraints be presented as a system of linear inequalities $\+A\xv\leq \bv$ or that a semidefinite constraint be expressed as $\sum_i \yv_i \+A_i \succeq \+C$. This may create severe difficulties for the user:
\begin{enumerate}
\item
Practical optimization involves more than just the optimization of an objective function subject to
constraints. Before optimization takes place, effort must be expended to formulate the model. 
This process of turning the intuition that defines the model ``on paper'' into a canonical form could be quite cumbersome. Consider the following example from graph isomorphism, see e.g.~\cite{atserias13siam}. Given two graphs $G$ and $H$, the LP formulation introduces a variable for every possible partial function mapping $k$ vertices of $G$ to $k$ vertices in $H$. In this case, it is not a trivial task to even come up with a convenient linear indexing of the variables, let alone expressing the resulting equations as $\+A\xv \leq \bv$. Such conversions require the user to produce and maintain complicated matrix generation code, which can be tedious and error-prone. Moreover, the reusability of such code is limited, as relatively minor modification of the equations could require large modifications of the code (for example, the user decides to switch from having variables over sets of vertices to variables over tuples of vertices). Ideally, one would like to separate the problem specification from the problem instance itself %({\todo  this sentence may go somewhere else}). 

\item Canonical forms are inherently propositional. By design they cannot model domains with a variable
number of individuals and relations among them without enumerating all of them. As already mentioned, however, many AI tasks and 
domains are best modeled in terms of individuals and relations. Agents must deal with heterogenous information of all types. 
Even more important, they 
must often build models before they 
know what individuals are in the domain and, therefore, before they know what variables exist. Hence modeling should 
facilitate the formulation of  abstract, general knowledge. 
\end{enumerate}
To overcome these downsides and triggered by the success of probabilistic logical languages, we show that optimization is 
liftable to the relational level, too. 
Specifically, we focus on linear programs which are the most tractable, best understood, and 
widely used in practice subset of mathematical programs. Here, the 
objective  is linear and the constraints involve linear (in)equalities only.  
Indeed, at the inference level within propositional models considerable 
attention has been already paid to the link between probabilistic models
and linear programs. This relation is natural since the MAP inference problem can be relaxed into linear programs. 
At the relational and lifted level, however, the link has not been established nor explored yet.

Our main contribution is
to establish this link, both at the language and at the inference level.
We introduce {\bf relational linear programming} best summarized as
\begin{align*}
%\big(\text{Lifting} + (\text{Logic} +   \text{LP} \mapsto \text{Model}) \mapsto \text{Model}\big) +  \text{Solver}:
\big((\text{LP} +  \text{Logic})  - \text{Symmetry}\big)+  \text{Solver}.
\end{align*}
The user describes a relational problem in a high level, 
relational LP modeling language and --- given a logical knowledge base (LogKB) encoding some individuals or rather data --- the system automatically induces a symmetry-reduced LP that in turn can be solved using any off-the-shelf LP solver. 
Its main building block are relational linear programs (RLPs). They are declarative LP templates defined through the logical concepts of individuals, relations, and quantified variables and allow the user to express the LP objective and constraints about a varying number of individuals without enumerating them. Together with a the LogKB referring to the individuals and relations, effectively a logical program consisting of logical facts and rules, it induces a ground LP. This ground LP
can be solved using any LP solver. 
Our illustrations of relational linear programming on several AI problems will showcase that relational programming
can considerably ease the process of turning the ``modeller's form" --- the form in which the modeler understands a problem or actually
a class of problems --- into a machine readable form since we can now deal with a varying number of individuals and relational among them in a declarative way. 
We will show this for computing optimal value function of 
Markov decision processes~\cite{littmanDK95}, for approximate inference within Markov logic 
networks~\cite{richardson2006markov} using LP relaxations as well as for collective classification~\cite{senNBGGE08}.

In particular, as another contribution, we will showcase a novel approach to {\bf collective classification by relational linear programming}. Say we want to classify people as having or not having cancer. In addition to the 
usual ``flat''  data about attributes of people like age, education and smoking habits, we have access to the social network among the people, cf. Fig.~\ref{fig:social}. This allows us to model influence among smoking habits among friend. 
%implicitly that a person is sometimes 
%passively smoking.
%, i.e., 
%if some of the friends of a particular person smoke. 
Now imagine that we want to do the classification using support vector machines (SVMs)~\cite{vapnik1998statistical} which boils down to a quadratic optimization problem. Zhou {\it et. al.}~\cite{zhou2002linear} have shown that the same problem can be modeled as an LP with only a small loss in generalization performance. 
Existing LP template language, however, would require feature engineering to capture smoking habits among friend. In contrast, in an RLP one simply adds rules
such as
\begin{equation*}
\mathtt{attribute(X, passive) \ \imp \ friends(X, Y), attribute(Y, smokes)}
\end{equation*}
saying that if two persons $\mathtt{X}$ and $\mathtt{Y}$ are friends and $\mathtt{Y}$ smokes, then $\mathtt{X}$ also smokes, at least 
passively. Moreover, as we will do in our experiments, we can formulate relational LP constraints to encode that objects that
link to each other tend to be in the same class. 
%Thus, using RLPs we can easily develop novel collective classification approaches just by programming. 
%\newsavebox{\tempbox}
\begin{figure}[t]
\captionsetup[subfigure]{labelformat=empty}
\footnotesize
\subfloat[]{
\begin{minipage}[c]{0.5\textwidth}
\begin{tikzpicture}[x=2.0cm,y=0.7cm]
\GraphInit[vstyle=Classic]
\begin{scope}[VertexStyle/.append style = {color=red,minimum size = 8pt, inner sep = 0pt}]
\Vertex[x=1,y=4,Lpos=90,L=$anna$]{a}
\Vertex[x=3,y=4,Lpos=90,L=$bob$]{b}
\Vertex[x=2,y=3,Lpos=45,L=$edward$]{e}
\Vertex[x=1,y=2,Lpos=180,L=$gary$]{g}
\Vertex[x=3,y=2,Lpos=0,L=$frank$]{f}
\Vertex[x=1,y=1,Lpos=-90,L=$iris$]{i}
\Vertex[x=3,y=1,Lpos=-90,L=$helen$]{h}

\Edge[style={bend left = 10}](a)(b)
\Edge[style={bend left = 10}](a)(e)
\Edge[style={bend left = 10}](e)(f)
\Edge[style={bend left = 10}](e)(g)
\Edge[style={bend left = 10}](g)(h)
\Edge[style={bend left = 10}](i)(h)

\end{scope}
\end{tikzpicture}
\end{minipage}}
\subfloat[]{
\begin{tabular} {l  c  c  c c}\hline
	\bf name	&	\bf age	&	\bf education	&	\bf smokes \\ \hline
	anna		&	27	&	uni		&	+ \\
	bob		&	22	&	college	&	- \\
	edward	&	25	&	college	&	- \\
	frank		&	30	&	uni		&	- \\
	gary		&	45	&	college	&	+ \\
	helen	&	35	&	school	&	+ \\
	iris		&	50	&	school	&	- \\ \hline
\end{tabular}}
\caption{Example for collective inference. There are $7$ people in a social network. Each person is described in terms of three attributes.
The class label ``cancer'' is not shown.\label{fig:social}}
\end{figure}

However,
%The compactness of the ground LP is our second contribution, and 
the benefits of relational linear programming go beyond modeling. 
%We will demonstrate that also show that linear programming can benefit even more from recent AIdevelopments. First, 
Since RLPs consist of templates that are instantiated several times to construct a ground linear model, they
are also likely to induce ground models that exhibit symmetries, and we will demonstrate how to detect and exploit them. Specifically, we will
introduce {\bf lifted linear programming} (LLP). It detects and eliminates symmetries in a linear program
in quasilinear time. Unlike lifted probabilistic inference methods such as lifted belief propagation~\cite{singla08aaai,kersting09uai,ahmadi2013mlj},  which works only with belief propagation approximations for probabilistic inference, 
LLP does not depend on any specific LP solver --- it can be seen as simply reparametrizing the linear program. As our experimental results on several AI tasks will show
%, together 
%with efficient grounding mechanisms similar to Tuffy~\cite{niu2011tuffy} 
this can result in significant efficiency gains.  

%, which has shown impressive speed-up compared to previous implementations. %{\todo We hopefully show we scale better than AMPL}. 

%From an optimisation point of view the benefit of the RLP representation is in the fact that it provides a fully declarative lifted definition of the problem which can easily be exploited by lifted solvers. %More concrete stuff here!!!!
%From a point of view of linear programming it allows to compactly represent high level structure of a program, separating it from a definition of a particular problem instance. This high level structure can then be exploited to speed up the inference. From the point of view of logic RLPs allow to write down optimisation problems over logical/relational domains in a very straight forward way. In particular machine learning tasks which can be expressed as LPs and require reasoning over relational data are easy to implement.

We proceed as follows. After touching upon related work, we start off by reviewing linear programming and existing LP template languages in Section~\ref{sec:linp}. 
Then, in Section~\ref{rlp}, we introduce 
relational linear programming, both the syntax and the semantics. Afterwards, Section~\ref{sec:lp} shows how to detect and exploit 
symmetries in linear programs. Before touching upon directions for future work and concluding, we illustrate relational linear programming 
on several examples from machine learning and AI. 
%Section \ref{lp}) and first-order logic (Section \ref{fol}). In Section \ref{rlp} we introduce the main result of our work - relational linear programming, including our approach to grounding. Section \ref{opt} talks about lifted optimisation approaches we apply to solve RLPs. We then give examples of some machine learning optimisation problems formulated as RLPs and present results of their evaluation (Section \ref{exp}). Finally we give direction for future work (Section \ref{fut}).

\section{Related Work on Relational and Lifted Linear Programming}
The present paper is a significant extension of the AISTATS 2012 conference paper~\cite{mladenov2012lifted}. 
It provides a much more concise development of lifted linear programming compared to~\cite{mladenov2012lifted} and 
the first coherent view on relational linear programming 
as a novel and promising way for scaling AI. To do so, it
develops the first relational modeling language for LPs and illustrates it empirically.  One of the advantages of the language is the closeness of its syntax to the 
mathematical notation of LP problems. This allows for a very concise and readable relational definition of linear optimization problems, which 
is supported by certain language elements from logic such as individuals, relations, and quantified variables.
The (relational) algebraic formulation of a model does not contain any hints how to process it. 
Indeed, several
modeling language for mathematical programming have been proposed. Examples of
popular modeling languages are AMPL~\cite{fourer1987ampl}, GAMS~\cite{brooke1992gams}, AIMMS~\cite{blomvall1993aimms}, and Xpress-Mosel~\cite{cirianiCH03}, but also see~\cite{kuip93,fragniere2002,wallace2005} for surveys. Although they are declarative, they focus on imperative
programming styles to define the index sets and data tables typically used to construct LP model. They do not
employ logical concepts such as clauses and unification. Moreover, since index sets and data tables are closely related to the
attributes and relations of relational database systems, there have been also proposals for ``feeding'' linear program directly from relational database systems, see 
e.g.~\cite{mitra95, atamturk2000relational, farrellM05}.
However, logic programming --- which allows to use e.g. compound terms --- was not considered and
the resulting approaches do not provide a syntax close to the 
mathematical notation of linear programs. Recently, Mattingley and Boyd~\cite{MattingleySB12} have
introduced CVXGEN, a software tool that takes a high level description of a convex optimization problem family, and automatically generates custom C code that compiles into a reliable, high speed solver for the problem family. Again concepts from logic programing were not used. 
Indeed, Gordon {\it et al.}~\cite{gordonHD09, zawadzkiGP11} developed
first-order programming (FOP) that combines the strength of mixed-integer linear programming and first-order logic. 
In contrast to the present paper, however, they focused on first-order logical reasoning and not on specifying relationally and solving efficiently 
arbitrary linear programs.  And non of these approaches as considered symmetries in LPs and how to detect and to exploit them.

Indeed, detection and exploiting symmetries within LPs  %is different from % with 
%While there has been research in 
is related to
symmetry-aware approaches %of symmetries 
in 
%LP done so far. Most of it has focused around 
{(mixed--)}integer programming~\cite{Margot_2010}. %However, as ILP and LP problems 
Howerver, they are vastly different to LPs in nature.
%, the present work differs notably from the state of the art in ILP symmetries (c.f. \cite{Margot_2010}). 
Symmetries in ILP are used for pruning the symmetric branches of search trees, thus the dominant paradigm is to add symmetry breaking inequalities, similarly to what has been done for SAT and CSP~\cite{Sellmann05}. In contrast, lifted linear programming achieves speed-up by reducing the problem size. 
For ILPs, symmetry-aware methods typically focus on pruning the search space to eliminate symmetric
solutions, see e.g.~\cite{Margot_2010} for a survey). In linear programming, however, one takes advantage of convexity
and projects the LP into the fixed space of its symmetry group~\cite{Boedi13}. The projections we investigate in
the present paper are similar in spirit. Until recently, discussions were mostly concentrated on the case
where the symmetry group of the ILP/LP consists of permutations, e.g.~\cite{Boedi10}. In such cases the problem
of computing the symmetry group of the LP can be reduced to computing the coloured automorphisms of
a ``coefficient'' graph connected with the linear program, see e.g.~\cite{berthold2008,Margot_2010}. Moreover, the reduction of the
LP in this case essentially consists of mapping variables to their orbits. Our approach subsumes
this method, as we replace the orbits with a coarser equivalence relation which, in contrast to the
orbits, is computable in quasilinear time. Going beyond permutations, B{\"o}di and Herr~\cite{Boedi13} extend the scope of symmetry,
showing that any invertible linear map, which preserves the feasible region and objective of the LP, may
be used to speed-up solving. While this setting offers more compression, the symmetry detection problem
becomes even more difficult. 

After the AISTATS conference paper \cite{mladenov2012lifted}, lifted
(I)LP-MAP inference approaches for (relational) graphical models based on graph automorphisms and variants have 
been explored in several ways, which go beyond the scope of the
present paper. We refer to~\cite{bui12arxive,noessner13aaai,mladenov14aistats,apsel14aaai}.

%!TEX root = main.tex
\section{Linear Programming}\label{sec:linp}

Linear programs, see e.g.~\cite{dantzig}, have found a wide application in the fields of operations research, where they are applied to problems like multicommodity flow and optimal resource allocation, and combinatorial optimization, where they provide the basis for many approximation algorithms for hard problems such as TSP.  They have also
found their way to machine learning and AI. Consider e.g. support vector machine (SVMs), which are among the most popular models for classification. Although they are traditionally formulated as quadratic optimization problems, there are also linear program (LP) formulations of SVMs such as Zhou {\it et al.}'s LP-SVMs~\cite{zhou2002linear}. Many other max-margin approaches use LPs for inference as well 
such as LP boosting~\cite{demiriz2002linear} or LP-based large margin structure prediction~\cite{wang2009large}. They 
have also been used for the decoding step within column generation approaches to solving quadratic problem formulations of the collective classification task, see e.g.~\cite{kleinBS08,torkamaniL13} and
reference in there. However, they are not based on LP-SVMs and on relational programming. In probabilistic graphical models LP relaxations 
are used for efficient approximate MAP inference, see e.g.~\cite{wainwright2008graphical}. Finding the optimal policy for a Markov decision problem can be formulated and solved with LPs~\cite{syed2008apprenticeship}. Likewise, they have been used for inverse reinforcement learning~\cite{ng2000algorithms} where the goal
is to estimate a reward function from observed behaviour. In addition, recent work use approximate LPs for relational MDPs~\cite{sannerB09}, which scale
to problems of previously prohibitive size by avoiding grounding. However, they were not using relational LPs.  Clustering can also
be formulated via LPs, see e.g.~\cite{komodakis2008clustering}. Ataman applied LPs to learning optimal rankings for binary classification problems \cite{ataman2007learning}. In many cases, even if a learning problem itself is not posed as an LP, linear programming is used to solve some intermediate steps. 
For instance Sandler {\it et al.}~\cite{sandler2005use} phrases computing the pseudoinverse of a matrix and greedy column selection from this pseudoinverse as LPs. The resulting matrix is then used for dimensionality reduction and unsupervised text classification. So what are linear programs?

\subsection{Linear Programs}
\label{lp}

A linear program (LP) is an optimization problem that can be expressed in the following general form:
\begin{align*}
	\operatorname*{minimize}\nolimits_{\xv\in\mathbb{R}^n} & \left< \cv, \xv\right> \\
	\text{subject to }\quad & \+A\xv \leq \bv \\
	 & \+G\xv = \hv\;,
\end{align*} 
where $\+A\in \mathbb{R}^{m\times n}$ and $\+G\in \mathbb{R}^{p\times n}$ are matrices, $\bv,\ \cv$ and $\hv$ are real vectors of dimension $m,n,$ and $p$ respectively,
and $\left< \cdot, \cdot\right>$ denotes the inner product of two vectors. 
% Whenever $c$ is the $0$-vector, our task reduces 
Note that the equality constraints can be reformulated as inequality constraints to yield an LP in the so-called dual form, 
\begin{equation}\label{eq:dualform}
\begin{aligned}
\operatorname*{minimize}\nolimits_{\xv\in {\mathbb R}^n } & \left< \cv, \xv \right>  \\ 
\text{ subject to }\quad &  \+A\xv\leq \bv\;,
\end{aligned}
\end{equation}
which we represent by the triplet $L = (\+A,\bv,\cv)$.
% \ignore{
% In general, a mathematical program has the following form:
% \begin{alignat*}{3}
% 	&\operatorname*{minimize}_{\sv \in \Omega}\quad& f(\sv) \\
% 	&\text{subject to }\quad & g(\sv) \leq \zero \\
% 	& &h(\sv) = \zero&\; . \\
% \end{alignat*} 
% ({\todo do we need this general definition})The goal is to select the variables ${\sv \in \Omega}$ so as to minimise the objective function ${f(\sv)}$ possibly subject to equality ${h(\sv) = \zero}$ and inequality ${g(\sv) \leq \zero}$ constraints. Mathematical programs are divided into several classes based on the set $\Omega$ and the classes of functions ${f}$, ${g}$ and ${h}$. Generally mathematical programs are NP-hard, but several restricted subclasses are tractable. The most efficient algorithms are available for the class of linear programs (LPs). A mathematical program is called linear if $\sv \in \mathbb{R}^n$ and functions ${f}$, ${g}$ and ${h}$ are linear in ${\sv}$. In fact, it is known that a class of LPs in in P \cite{khachiyan1980polynomial}.
% }
% Additional interest to the field has been recently attracted by research on lifted linear programming \cite{mladenov2012lifted}. The idea is to read off equivalence relationships between variable from an LP and solve a potentially much smaller problem with the same optimum. This approach has been shown to be especially effective on many machine learning problems due to the structural symmetry they naturally possess.

While LP models often look intuitive on paper, applying these models in practice presents a challenge. The main issue is that the form of a problem representation that is natural for most solvers (i.e. the L triplet representation) is not the form that is natural for domain experts. 
%A canonical linear program is defined by a matrix and 2 vectors:
%\begin{align*}
%		\min & \  \cv^T \xv \\
%		\text{subject to } & A \xv \leq \bv 
%\end{align*} 
%where we have omitted the range of $\xv$ for the sake of simplicity. 
Furthermore, the matrix $\+A$ is typically sparse, i.e., having mostly $0$-entries. Modeling any real world problem in this form can be quite error prone and time consuming. In addition, it is often necessary to separate the general structure of a problem from the definition of a particular instance. For example, a flow problem formulated as an LP consists of a set of constraint for each edge in a graph, which do not depend on a particular graph. Hence, the definition of the flow LP can be separated from the specification of the graph at hand and, in turn, be applied to different graphs.

\subsection{Declarative Modelling Languages for Linear Programs}
The problems above are traditionally solved by modelling languages. They simplify LP definition by allowing to use algebraic notation instead of matrices and define an objective and constraints through parameters whose domains are defined in a separate file, thus enabling model/instance separation. 
Starting from \eqref{eq:dualform}, they typically make the involved arithmetic expressions explicit:
%That is, a linear program in this form --- also called general LP --- looks like this: {\todo MM: I don't really understand what is being said here}
\begin{equation}\label{eq:setlp}
\begin{aligned}
		\operatorname*{minimize}\nolimits_{\+x\in \mathbb{R}^n} & \sum\nolimits_{j \in P} c_j x_j \\
		\text{subject to }\quad & \sum\nolimits_{j \in P} a_{ij} x_j \leq b_i \text{ for each } i \in K\;,
\end{aligned} 
\end{equation}
where the sets ${P}$ and ${K}$ as well as the  corresponding non-zero entries of vectors ${\cv}$, ${\bv}$ and matrix ${\+A}$ are defined in a separate file. This simplified representation is then automatically translated to the matrix format and fed to a solver that the user prefers. 

To code the LP in this ``set form'', several mathematical programming modelling languages have been proposed to implement this general idea. According to NEOS solver statistics\footnote{http://www.neos-server.org/neos/report.html; accessed on April 19, 2014.}, AMPL 
is the most popular one.  We only briefly review the basic AMPL concepts. For more details, we refer to ~\cite{fourer1987ampl}. 

Based on
the ``set form'', an LP can be written in AMPL as shown in Fig.~\ref{lp:ampl}.
\begin{figure}[t]
\begin{lstlisting}[language=ampl]
set P;   #column dimension of A 
set K;   #row dimension of A
param a {j in P, i in K};   #provided as input
param c {j in P};           #provided as input
param b {i in K};           #provided as input
var x {j in P};  #determined by the solver

#the objective
minimize:  sum {j in P} c[j] * x[j];   
#the constraints
subject to sum {j in P} a[i, j]*x[j] <= b[i]; 
\end{lstlisting}
\caption{AMPL declaration scheme for a linear program in ``set form'' as shown in~\eqref{eq:setlp}.\label{lp:ampl}}  
\end{figure}
In principle an AMPL program consists of one objective and a number of ground or indexed constraints. If a constrain is indexed, (i.e. the constraint in the example above is indexed by the set K) a ground constraint is generated for every combination of values of index variables (in the example above there is just one index variable in the constrain, hence a ground constraint is generated for every value in K). The keyword \textbf{set} declares a set name, whose members are provided in a separate file. The keyword \textbf{param} declares a parameter, which may be a single scalar value or a collection of values indexed by a set. Subscripts in algebraic notation are written in square brackets as in {\tt b[i]} instead of ${b_i}$. The values to be determined by the solver are defined by the \textbf{var} keyword. The typical ${\sum}$ symbol is replaced by the \textbf{sum} keyword. The key element of the AMPL system is the so called {\bf indexing expression} 
\begin{equation*}
\mathtt{ \{j \ in \ P\}}\;. 
\end{equation*}
In addition to being a part of variable/parameter declaration, they serve both as limits for sums and as indices for constraints. Finally,
comments in AMPL start with the symbol {\tt \#}.

In relational linear programs, which we will introduce next, we are effectively mixing first order logic into AMPL. This allows us to keep AMPL's benefits that make it the number one choice for optimization experts and at the same time enable the representation of relational problems.
% over relational domains and opening perspective for even more efficient solvers that exploit the structure of a problem. 

%\input{fol}

%!TEX root = main.tex

\section{Relational Linear Programming}
\label{rlp}
The main idea of relational linear programming is to parameterize AMPL's arithmetic expressions by logical variables 
and to replace AMPL's indexing expression 
by queries to a logical knowledge base. 
Before showing how to do this, let us briefly review logic programming. For more details we refer to \cite{lloyd1987foundations,flach94,deraedt08springer}.

\subsection{Logic Programming}
\label{fol}
A logic program is a set of clauses constructed using four types of symbols: constants, variables, functors, and predicates. 
Reconsider the collective classification example from the introduction, also see in Fig.~\ref{fig:social}, and in particular the ``passive smoking'' rule.
 Formally speaking, we
have that $\mathtt{attribute/2}$
$\mathtt{friends/2}$ are {\bf predicates } (with their {\em arity}, i.e.,
number of arguments listed explicitly). The symbols  \texttt{anna, bob, edward, frank, gary, helen, iris}
are {\bf constants} and $\mathtt{X}$, and $\mathtt{Y}$ 
are {\bf variables}.
All constants and variables are also {\bf terms}
In addition, one can also have structured terms such as
$\mathtt{s(X)}$, which contains the {\bf functor} $\mathtt{s/1}$ of
arity $1$ and the term $\mathtt{X}$.  
\emph{Atoms} are predicate
symbols followed by the necessary number of terms, e.g.,
$\mathtt{friends(bob,anna)}$, $\mathtt{nat(s(X))}$,
$\mathtt{attribute(X,passive)}$, etc. {\bf Literals} are atoms $\mathtt{nat(s(X))}$ (positive literal)
and their negations $\mathtt{not \ nat(s(X))}$ (negative literals).
We are now able to define the key concept of a {\bf clause}.
They are formulas of the form
\begin{equation*}
\mathtt{A \imp B_1,\ldots ,B_m}
\end{equation*}
where $\mathtt{A}$ -- the head -- and the $\mathtt{B_j}$ --- the body --- are logical
atoms and all variables are understood to be universally
quantified. For instance, the clause
%\begin{equation*}
\[c\equiv{\mathtt{attribute(X,passive) \imp friends(X,Y),attribute(Y,smokes)}}\]
%\end{equation*}
can be read as $\mathtt{X}$ has $\mathtt{attribute}$  $\mathtt{passive}$ 
if $\mathtt{X}$ and $\mathtt{Y}$ are
$\mathtt{friends}$ and $\mathtt{Y}$ has the $\mathtt{attribute}$ $\mathtt{smokes}$.  
%We call $\mathtt{attribute(X,passive)}$ 
%the $\head(c)$ of this clause, and ${\mathtt{friends(X,Y), attribute(Y,smokes)}}$ 
%the $\body(c)$. 
Clauses with an empty body are {\bf facts}.
A  {\bf logic program} consists of a finite set of clauses.  
The set of variables in a term, atom, conjunction or 
clause $E$, is denoted as $\Var(E)$, e.g., $\Var(c) = \{\mathtt{X},\mathtt{Y}\}$.
A term, atom or clause $E$ is {\bf ground} when there is no variable occurring in
$E$, i.e. $\Var(E)=\emptyset$. 
A clause $c$ is {\bf range-restricted} when all variables in the left-hand side of $\imp$ also 
%in the head of the
%clause also 
appear in the right-hand side. %body of the clause. %, i.e., $\Var(\head(c))\subseteq\Var(\body(c))$.

A {\bf substitution} ${\theta=\{V_1/t_1,\ldots ,V_n/t_n\}}$, e.g. $\{\mathtt{Y}/\mathtt{anna}\}$, is an
assignment of terms $t_i$ to variables $V_i$. Applying a
substitution $\theta$ to a term, atom or clause $e$ yields the
instantiated term, atom, or clause $e\theta$ where all occurrences
of the variables $V_i$ are simultaneously replaced by 
$t_i$, e.g. $c\{\mathtt{Y}/\mathtt{anna}\}$ is 
%\begin{equation*}
%c^\prime\equiv
\[{\mathtt{attribute(X,passive) \imp friends(X,anna),attribute(anna,smokes)}}\;.\]
%\end{equation*}
The {\bf Herbrand base} of a logic program
$P$, denoted as $\hb(P)$, is the set of all ground atoms
constructed with the predicate, constant and function symbols in
the alphabet of $P$. A {\bf Herbrand interpretation} for a logic program $P$ is a subset of $\hb(P)$. 
A Herbrand interpretation $I$ is a {\bf model} of a clause $c$
if and only if for all substitutions $\theta$
such that $\body(c)\theta \subseteq I$ holds, it also holds that $\head(c)\theta \in I$.
%The interpretation $I$ is a model of a logic program $P$ if $I$ is
%a model of all clauses in $P$. 
A clause $c$ (logic program $P$) {\bf entails}
another clause $c^\prime$ (logic program $P^\prime$), 
denoted as
$c\models c^\prime$ ($P\models P^{\prime}$),
if and only if, each model of $c$ ($P$) is 
also a model of $c^\prime$ ($P^{\prime}$).

The {\bf least Herbrand model} $\lhm(P)$, which constitutes the semantics of the logic
program $P$, consists of all facts $f \in\hb(P)$ such that $P$
logically entails $f$, i.e. $P \models f $.  
Answering a {\bf query} $q\equiv\mathtt{\imp G_1,G_2\ldots, G_n }$  with respect to a logic program is to 
determine whether the query is entailed by the program or not. That is, 
$q$ is true in all worlds where $P$ is true.
This is often done by refutation: $P$ entails $q$ iff ${P\cup \lnot q}$ is unsatisfiable.

Logic programming  is especially convenient for representing relational data like the social graph in Fig.~\ref{fig:social}. All one needs is the binary predicate 
{\tt friend/2} to encode the edges in the social graph as well as the
predicate ${\tt attribute(X, Attr)}$ to code the attributes of the people in the social network. %In general, 
%concepts such as \textit{has friends who smoke} or even \textit{shortest path to a smoking friend} can be easily expressed with clauses.

\subsection{Relational Linear Programs}
Since our language can be seen as a logic programming variant of AMPL, we introduce its syntax in a contrast to the AMPL syntax. To do so, let us consider a well known network flow problem \cite{ahuja1993network}. The problem is to, given a finite directed graph ${G(V, E)}$ in which every edge ${(u, v) \in E}$ has a non-negative capacity ${c(u,v)}$, and two vertices ${s}$ and ${t}$ called source and sink, maximize a function $f: V\times V \rightarrow \mathbb{R}$ called flow with the first parameter fixed to ${s}$ (outgoing flow from the source node), subject to the following constraints: $f(u, v) \leq c(u, v),  \ f(u, v) \geq 0,$ and
\begin{align*}
	\sum\nolimits_{w \in V/\{s, t\}} f(u, w) = \sum\nolimits_{w \in V/\{s, t\}} f(w, u) \;,
\end{align*}
where the third constraint means that incoming flow equals to outgoing flow for internal vertices. Such flow problems can naturally be formulated 
as an LP specified in  AMPL as shown in Fig.~\ref{lp:flow}.
\begin{figure}[t]
\begin{lstlisting}[language=ampl]
set VERTEX;                                                      #vertices 
set EDGES within (VERTEX diff {sink}) cross (VERTEX diff {source}); #edges 

param source symbolic in VERTEX;                    #entrance to the graph
param sink symbolic in VERTEX, <> source;           #exit from the graph
param cap {EDGES} >= 0;                             #flow capacities

var Flow {(i,j) in EDGES} >= 0, <= cap[i,j];          #flows

maximize: sum {(source,j) in EDGES} Flow[source,j];    #objective

subject to {k in VERTEX diff {source,sink}}:          #conservation of flow
  sum {(i,k) in EDGES} Flow[i,k] = sum {(k,j) in EDGES} Flow[k,j];
\end{lstlisting}
\caption{AMPL specification for a general flow linear program. The vertices and edges as well as the corresponding flow capacities are provided in separate files. The flows of the edges are declared to be determined by the LP solver.\label{lp:flow}}
\end{figure}
The program starts with a definition of all sets, parameters and variables that appear in it. They are then used to define the objective and the constraints from the network flow problem, in particular the third one. The first two constraints are incorporated into the variable definition.  As one can see, AMPL
allows one to write down the problem description in a declarative way. It frees the user from engineering instance specific LPs while capturing the 
general properties of the problem class at hand. However, AMPL does not provide logically parameterized definitions for the arithmetic expressions and for the
index sets. RLPs, which we will introduce now, feature exactly this. 

A first important thing to notice is that AMPL mimics arithmetic notation in its syntax as much as possible. It operates on sets, intersections of sets and arithmetic expressions indexed with these sets. Our language for relational linear programming 
effectively replaces these constructs with logical predicates, clauses, and queries to define the three main parts of an RLP: the objective template, the constraints template, and a logical knowledge base. 
An RLP for the flow example is shown in Fig.~\ref{rlp:flow}. It directly codes the flow constraints,
concisely captures the essence of flow problems, and illustrates nicely that linear programming in general can be viewed as being highly relational
in nature. Let's now discuss this program line by line. 
\begin{figure}[t]
\begin{lstlisting}[language=ampl]
var flow/2;              #the flow along edges is determined by the solver
               
outflow(X) = sum {edge(X, Y)} flow(X, Y);                #outflow of nodes
inflow(Y)  = sum {edge(X, Y)} flow(X, Y);                #outflow of nodes

maximise: sum {source(X)} outflow(X);                     #objective

subject to {vertex(X), not source(X), not sink(X)}:  #conservation of flow
  outflow(X) - inflow(X) = 0;
subject to {edge(X, Y)}: cap(X, Y) - flow(X, Y) >= 0;      #capacity bound 
subject to {edge(X, Y)}: flow(X, Y) >= 0;               #no negative flows
\end{lstlisting}
\caption{Relational encoding the general flow LP. For details we refer to the main text.\label{rlp:flow}}
\end{figure}

Predicates define variables and parameters in the LP. In the flow example, {\tt flow/2} captures for example the flows between nodes. Sets that are explicitly defining domains in AMPL are discarded and parameter/variable domains are defined implicitly.
In contrast to logic, (ground) atoms can take any numeric value, not just true or false.
For instance the capacity between the nodes is captured by
{\tt cap/2}, and the specific capacity between node $\mathtt{f}$ and $\mathtt{t}$
could take the value $3.7$, i.e., {\tt cap(f, t) = 3.7}. 
Generally, atoms are parameters of the LP.  To declare that
they are values to be determined by the solver we follow AMPL's notation,
\begin{lstlisting}[language=ampl,frame=none,basicstyle=\footnotesize\ttfamily]
var flow/2;
\end{lstlisting}
%That also means that there will be no definitions for {\tt flow/2}  since it values are do be determined by the solver. In contrast,
The in- and outflows per node, {\tt inflow/2} and {\tt outflow/2},  are defined within the RLP. They are the sums of all flows into 
respectively out of a node. To do so, we use logically {\bf parameterized equations} or {\bf par-equations} in short. 
A par-equation is a finite-length expression of the form 
\begin{equation*}
\phi_1 \ = \ \phi_2\;,
\end{equation*}
where $\phi_1$ and $\phi_2$ are {\bf par-expressions} of the form
\begin{equation*}
\mathtt{sum} \{ \phi \} \ \psi_1 \  \mathtt{op}_1 \ \psi_2 \ \mathtt{op}_2 \ \ldots  \ \mathtt{op}_{n-1} \ \psi_n
\end{equation*}
of finite length. Here $\psi_i$ are numeric constants, atoms or {par-expressions}, and the $\mathtt{op}_j$ are arithmetic operators. 
The term 
$\mathtt{sum} \{ \phi \} $ --- which is optional --- essentially implements the AMPL aggregation {\bf sum} but now indexed over a logical query $\phi $. That is, the AMPL
indexing expression {\tt \{j in P\}} for the aggregation is turned into an indexing over all tuples in the answer set of the logical query. 
Essentially, one can think of this as calling the Prolog meta-predicate 
\begin{equation*}
\mathtt{setof\big(}\Var(\phi,\psi_1,\ldots,\psi_n),(\phi,\psi_1,\psi_2,\ldots,\psi_n) ,\mathtt{P\big)}
\end{equation*} 
treating the par-expression $\psi_1 \  \mathtt{op}_1 \ \psi_2 \ \mathtt{op}_2 \ \ldots  \ \mathtt{op}_{n-1} \ \psi_n$ as a conjunction $\psi_1,\psi_2,\ldots,\psi_n$.
This will produce the set {\tt P} of all substitutions of the variables $\Var(\phi,\psi_1,\ldots,\psi_n)$
%{\tt X}  
(with any duplicate removed) such that the query $\phi,\psi_1,\ldots,\psi_n$ is satisfied. In case, we are interested in multi-sets, i.e., to express counts, one may also use $\mathtt{findall}/3$. This could be expressed using $\mathtt{sum} \langle\phi\rangle $ instead of $\mathtt{sum} \{ \phi \}$ . The {\tt sum} aggregation and the involved par-expression is then evaluated over the resulting multidimensional index {\tt P}. If no {\tt sum} is provided, this will just be logical indexing for the evaluation of the  par-expression $\psi_1 \  \mathtt{op}_1 \ \psi_2 \ \mathtt{op}_2 \ \ldots  \ \mathtt{op}_{n-1} \ \psi_n$.
Finally, we note that all par-equalities are assumed implicitly to be all-quantified. That is, they may lead to several ground instances, in particular if there
are free variables in the logical query $\phi$.

With this at hand, we can define {\tt inflow/2} and {\tt outflow/2} as follows
\begin{lstlisting}[language=ampl,frame=none,basicstyle=\footnotesize\ttfamily,firstnumber=3]
outflow(X) = sum {edge(X, Y)} flow(X, Y);
inflow(Y)  = sum {edge(X, Y)} flow(X, Y);
\end{lstlisting}
Since {\tt Y} is bounded by the summation for {\tt outflow/1} and {\tt X} by the summation for {\tt inflow/1}, this says that there are
two equality expressions per node {\tt X} --- one for the outflow and one for the inflow --- summing over all flows of the out- respectively 
incoming edges of the node {\tt X}.
Indeed, {\tt edge/2} is not defined in the flow RLP. In such cases,
we assume it to be defined within a logical knowledge base LogKB (see below) represented as logic program.
%Generally, par-equations are implicitly assumed to be all-quantified. 
%Since {\tt Y} is bounded by the {\tt sum} statement, we get one equality statement per node. 
%, and the statement
%{\tt cap(f, t) = 3.7} can be viewed as short-hand for {\tt cap(f, t, 3.7)}. In turn,  by overloading notation,
% {\tt cap/2} can be queried in a logical way. The query $\imp\mathtt{cap(t,f)}$ is turned into $\imp\mathtt{cap(t,f,\_)}$m where $\_$ denotes
% an anonymous variable. 
 %With this in mind, the two par-expressions for {\tt inflow/2} and {\tt outflow/2} say that the in- and outflows are the sums of all flows into respectively out of a node. 
%Since the only freeTo avoid multiple, conflicting value assignments in par-expressions, we assume that all free variables that do not appear in 
%both $\phi_1$ and $\phi_2$ are bounded by some aggregation statement. 

Similarly, we can now define the objective\footnote{For the sake of simplicity, we here assume that there are exactly one source and one sink vertex. If one wants to enforce this, we could simply add this as a logical constraint within the selection query, resulting in an empty objective if there are several source and sink notes. In AMPL, one would use additional {\tt check}  statements to express such restrictions, which cannot be expressed using simple inequalities.} using a par-expression:
\begin{lstlisting}[language=ampl,frame=none,basicstyle=\footnotesize\ttfamily,firstnumber=6]
maximise: sum {source(X)} outflow(X);
\end{lstlisting}
This says, that we want to maximize the outflows for all source nodes.  Note that we assume that all variables appear in 
the {\tt sum} statement to avoid producing multiple and conflicting objectives. 

Next, we define the constraints. Again we can use par-equations or actually par-inequalities. {\bf Par-inequalities}
are like par-equations where we use inequalities instead of equalities. For the flow example they are:
\begin{lstlisting}[language=ampl,frame=none,basicstyle=\footnotesize\ttfamily,firstnumber=8]
subject to {node(X), not source(X), not sink(X)}: 
  outflow(X) - inflow(X) = 0;
subject to {edge(X, Y)}: cap(X, Y) - flow(X, Y) >= 0;
subject to {edge(X, Y)}: flow(X, Y) >= 0;
\end{lstlisting}
This again illustrates the power of RLPs.  Since indexing expressions are logical queries, we can naturally express things that either look cumbersome in AMPL or even go beyond its capabilities. 
For instance the concept of {\bf internal edge}, which is explicitly represented by a lengthy expression with sets intersections and differences in AMPL, is implicitly represented by a combination of the indexing query {\tt node(X), not source(X), not sink(X)} and par-equations {\tt in}-/{\tt outflow}.

\begin{figure}[t]
\begin{center}
\begin{tikzpicture}[x=2.0cm,y=0.8cm]
\GraphInit[vstyle=Classic]
\begin{scope}[VertexStyle/.append style = {color=gray, minimum size = 8pt, inner sep = 1pt}]
\Vertex[x=2,y=1.5,Lpos=90,L=$a$]{a}
\Vertex[x=2,y=-1.5,Lpos=-90,L=$b$]{b}
\Vertex[x=4,y=1.5,Lpos=90,L=$c$]{c}
\Vertex[x=4,y=-1.5,Lpos=-90,L=$d$]{d}
\end{scope}

\begin{scope}[VertexStyle/.append style = {color=red,minimum size = 8pt, inner sep = 1pt}]
\Vertex[Lpos=180,L=$s$]{s}
\Vertex[x=6,y=0,Lpos=0,L=$t$]{t}

\tikzset{EdgeStyle/.style={postaction=decorate,decoration={markings,mark=at position 0.5 with {\arrow{triangle 60}}}},
MyLabel/.style={
   auto=left,sloped,color=blue,
   fill=none,
   outer sep=0.2ex}}

\Edge[style={bend left = 10},label={$4$},labelstyle={MyLabel}](s)(a)
\Edge[style={bend right = 10},label={$2$},labelstyle={MyLabel}](s)(b)
\Edge[style={bend left = 10},label={$3$},labelstyle={MyLabel}](a)(c)
\Edge[style={bend right = 25},label={$2$},labelstyle={MyLabel}](b)(c)
\Edge[style={bend right = 10},label={$3$},labelstyle={MyLabel}](b)(d)
\Edge[style={bend right = 25},label={$1$},labelstyle={MyLabel}](c)(b)
\Edge[style={bend left = 10},label={$2$},labelstyle={MyLabel}](c)(t)
\Edge[style={bend right = 10},label={$4$},labelstyle={MyLabel}](d)(t)

\end{scope}
\end{tikzpicture}
\end{center}
\caption{A graph for a particular instance of the flow problem. The node $s$ denotes the source, and $t$ the sink. The numbers associated with
the edges are the flow capacities. \label{fig:flowinst} }
\end{figure}
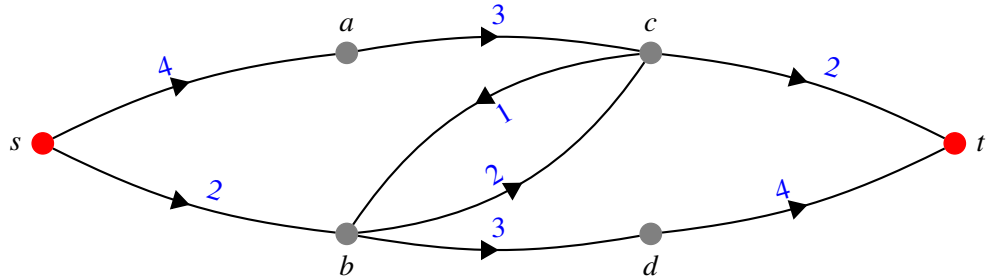
Finally, as already mentioned, everything not defined in the RLP %parameter atoms %i.e. grounding of parameter atoms, 
is assumed to be defined in an external logical knowledge base LogKB. It includes groundings of parameter atoms and definitions of intensional predicates (clauses). 
%Groundings of predicates used in an LP template (parameters in AMPL terminology) are defined in the logical knowledge base (LogKB), 
We here use Prolog but assume that each query from the RLP produces a finite set of answers, i.e., ground substitutions of its logical variables. 
For instance the LogKB for the instance of the flow problem shown in Fig.~\ref{fig:flowinst} can be expressed as follows: % like this:
\begin{lstlisting}[language=ampl,frame=none,basicstyle=\footnotesize\ttfamily,numbers=none,backgroundcolor=\color{blue!10}]
cap(s,a) = 4.  cap(s,b) = 2.  cap(a,c) = 3.  cap(b,c) = 2.  
cap(b,d) = 3.  cap(c,b) = 1.  cap(b,t) = 2.  cap(d,t) = 4.

edge(X,Y) :- cap(X,Y).

vertex(X) :- edge(X,_).
vertex(X) :- edge(_,X).

source(s).
sink(t).
\end{lstlisting}
where {\tt cap(s, a) = 4} is short-hand notation for {\tt cap(s,a,4)} and {\tt cap(X, Y)} for {\tt cap(X,Y,\_)}, where we use an anonymized variable `{\tt \_}'. 
Predicates {\tt edge} and {\tt vertex} are defined intensionally using logical clauses. 
By default intensional predicates take value 1 or 0 (corresponding to true and false) with the RLP. 
Values of intensional predicates are computed before grounding the RLP. 

Indeed, querying for a vertex would results in a multi-set, since the definition of {\tt vertex/1}
tests for a vertex as the source and the sink of an edge. However, recall that a logical indexing statement {\tt \{...\}} removes any duplicate first. 
Consequently, we could have used directly the capacity statements {\tt cap/2} to define edges and vertices. 
\begin{lstlisting}[language=ampl,frame=none,basicstyle=\footnotesize\ttfamily,firstnumber=8]
subject to {cap(X,_), not source(X), not sink(X)}: 
  outflow(X) - inflow(X) = 0;
subject to {cap(X, Y)}: cap(X, Y) - flow(X, Y) >= 0;
subject to {cap(X, Y)}: flow(X, Y) >= 0;
\end{lstlisting}
This works, since the logical indexing statement {\tt \{...\}} produces a set variable bindings so that the 
query {\tt cap(X, \_)} to a knowledge base will return {\tt a} only once (note: we do not allow multiset indexing $\langle\ \ldots \rangle $  here, as it only generates redundant constraints). And, due to the use of the anonymous variable {\tt \_}, 
its values of the corresponding variable are not included in the query results. 

In any case,
using logical clauses within relational linear programs is quite powerful. For example, a passive smoking predicate for the collective classification example from
the introduction can be defined in the following way:
\begin{lstlisting}[language=ampl,frame=none,basicstyle=\footnotesize\ttfamily,numbers=none,backgroundcolor=\color{blue!10}]
attribute(X, passive) :- friends(X, Y), attribute(Y, smokes).
\end{lstlisting}
It can then be used within the RLP. 
% LP template like any other parameter predicate. Quite often a counting functionality is also required. For instance, the user might
%want to have a predicate representing the number of smoking friends a person has. This can be achieved by adding the keyword {\tt count} to the beginning 
%of a query.
%\begin{minted}[fontsize = \small,firstnumber=8]{rlp}
%attribute(X, passive) :- count, friends(X, Y), attribute(Y, smokes).
%\end{minted}
%The predicates above takes a value equal to the number of smoking friends of a person {\tt X}. In RLP logical rules support conjunctions, disjunctions and negations but don't support parenthesised expressions, recursion or advanced constructs like lists. This restriction allows us to apply efficient grounding techniques described later.
Or, as another example, consider the compact representation of MLNs for instance for MAP LP inference using RLPs. 
The following logKB encodes compactly the smokers MLN \cite{richardson2006markov}
%In many real world situation a lot of ground predicates take the same value and typing them all is very tedious. To mitigate this problem we introduce typed variables in the LogKB, inspired by a similar construct in MLNs. For example, consider a MAP LP \cite{wainwright2008graphical} for the smokers MLN \cite{richardson2006markov}.
 This MLN has the following weighted clause: 
 %\begin{equation*}
 $0.75 \ \mathtt{smokes(X) \imp cancer(X)}\;,$ 
 %\end{equation*}
 which means that smoking leads to cancer, and our belief in this fact is proportional to $0.75$. In the MAP LP we want to have a predicate with value $0.75$ for every instance of a rule which is true, and a predicate with value 0 for every instance which is false. Instead of writing this down manually, we can use:
%\begin{minted}[linenos,
%	      fontsize = \small,
%               numbersep=5pt]{rlp}
%dom Person = [anna, bob]
%dom Val = [0, 1]
%
%def smokes(Person)
%def cancer(Person)
%def w(_, _, Val, Val)
%
%w(smokes(X), cancer(X), 1, 0) = 0
%w(smokes(X), cancer(X), V1, V2) = 0.75
%\end{minted}
\begin{lstlisting}[language=ampl,frame=none,basicstyle=\footnotesize\ttfamily,numbers=none,backgroundcolor=\color{blue!10}]
person(anna).   person(bob). ...

value(0).  value(1).

w(smokes(X), cancer(X), 1, 0) = 0 :- person(X).
w(smokes(X), cancer(X), V1, V2) = 0.75 :- 
                           person(X), value(V1), value(V2).
\end{lstlisting}
Please keep our short-hand notation in mind: {\tt w(smokes(X), cancer(X), 1, 0) = 0} stands for {\tt w(smokes(X), cancer(X), 1, 0, 0)}. 
This LogKB will generate the following ground atoms for the weights:
\begin{lstlisting}[language=ampl,frame=none,basicstyle=\footnotesize\ttfamily,numbers=none,backgroundcolor=\color{blue!10}]
w(smokes(anna), cancer(anna), 1, 0) = 0
w(smokes(anna), cancer(anna), 0, 0) = 0.75
...
w(smokes(bob), cancer(bob), 1, 1) = 0.75
\end{lstlisting}
To summarize, a {\bf relational linear program} (RLP) consists  of 
\begin{itemize} 
\item variable declarations of predicates, 
\item one par-equality to define the objective, and 
\item several par-(in)equalities to define the constraints. 
\end{itemize}
Everything not explicitly defined is assumed to be a parameter defined in an external LogKB. 
We now show that any RLP induces a valid ground LP.
\begin{theorem}\label{th:semantics}
An RLP together with a LogKB (such that the logical queries in the RLP have finite answer-sets) induces a ground LP.
\end{theorem}
\begin{proof}
The intuition is to treat par-(in)equality as a logical rules, treating the arithmetic operators as conjunctions, the
(in)equalities as the $\imp$, and {\tt sum} statements as meta-predicates. Then the finiteness follows from
the assumption of finite answer-sets for logical queries to LogKB. Finally, each ground clause can be turned back into
an arithmetic expression resp. (in)equality. 
More specifically, 
a predicate is grounded either to numbers or to LP variables. Since par-expressions
are of finite length, we can turn them into a kind of prenex normal form, that ist, we can write them as strings of 
$\mathtt{sum} \{ \phi \}$ statements followed by a {\tt sum}-free part where bracket expressions are correspondingly simplified.  
Now, we can ground a par-expression inside-out with respect to the {\tt sum} statements. Each of these groundings is finite 
due to the assumption that the queries to the LogKB have finite sets of answers. In turn, since the objective assumes that all 
variables are bounded by the  {\tt sum} statements in the prenex normal form, only a single ground sum over
numbers and LP variables is produced as objective. For par-(in)equalities encoding constraints both sides of
the (in)equality are par-expressions. Hence, using the same argument as for the objective, they produce at most a finite number of
ground (in)equalities. Taking everything together, an RLP together with a finite LogKB always induces a valid ground LP.
\end{proof}
For illustration reconsider the flow instance in Fig.~\ref{fig:flowinst}. It induces the following 
ground LP for the relational flow LP in Fig.~\ref{rlp:flow} (for the sake of readability, only two groundings are shown for each constraint for compactness).
\begin{lstlisting}[language=ampl,frame=none,basicstyle=\footnotesize\ttfamily]
maximize: flow(s,a) + flow(s,b);

subject to: flow(b,c)+flow(b,d)-flow(s,b)-flow(c,b)=0;
subject to: flow(a,c) - flow(s, a) = 0
...
subject to: 3 - flow(a, c) >= 0;
subject to: 2 - flow(c, t) >= 0;
...
subject to: flow(c, t) >= 0;
subject to: flow(b, c) >= 0;
...
\end{lstlisting}
\begin{algorithm}[t]
\KwIn{RLP together with a LogKB}
\KwOut{Ground LP $G$ consisting of ground (AMPL) statements}
Set $G$ to the empty LP\;
``Flatten'' par-(in)equalities and the objective into prenex normal form by inlining aggregates and simplifying brackets\;
	\For{each par-(in)equality {\bf and} the objective}{
		\For{each sum-aggreation and each separate atom} {
			Query the LogKB in order to obtain a grounding or a set of groundings if we are dealing with a constraint which involves an indexing expression\;
		}
		Concatenate results of queries evaluation for each grounding to form a ground (in)equality resp.~objective\;
		Add the ground (in)equality resp.~objective to $G$\;
	}
\Return{$G$}	
\caption{Grounding RLPs. The resulting ground LP $G$ can be solved using any LP solver after transforming into the solver's input form. This can of course be automated.\label{gag}}
\end{algorithm}

%\subsection{Grounding}
But how do we compute this induced ground LP, which will then become an input to an LP solver? 
That is, given a RLP  and a LogKB how do we effectively expand all the sums and substitute all the parameters with numbers?  
The proof of Theorem~\ref{th:semantics} essentially tells us how to do this. We treat all par-(in)equalities and par-expressions
as logical rules (after turning them into prenex normal form and simplifying them). Then any Prolog engine can be used to
compute all groundings. A simple version of this is summarized in Alg.~\ref{gag}. %, any Prolog engine could be used. 

In our experiments, however, we have used a more efficient approach similar to the one 
introduced in Tuffy~\cite{niu2011tuffy}. The idea is to use a
relational database management system (RDBMS) for bottom-up grounding, first populating the ground predicates into a DB and then translating logical formulas into SQL queries. Intuitively, this allows to exploit RDBMS optimizers, and thus significantly speed up of the grounding phase. 
%Compared to naive able to achieve a speed up from 7 hours to 2 minutes on some tasks.
%\begin{enumerate}
%\item All the ground predicates from LogKB are populated into an RDBMS
%\item If LogKB contains intensional predicates they are precomputed and populated into an RDBMS as well
%\item The objective and the constraints are "flattened" by substituting aggregate references with aggregate bodies and simplifying bracketed expressions
%\item At this point objective and constraints consist of sums of sum expressions. Each such expression is converted into an SQL query to the db. 
%\item PostgreSQL is used in the current implementation, which allows to perform arithmetic computations and string concatenations inside an SQL request, so we are able to get sums of LP variables with corresponding coefficients directly from SQL requests. The only post processing needed is concatenation of these strings.
%\end{enumerate}
We essentially follow the same strategy for grounding RLPs. PostgreSQL is used in the current implementation, which allows to perform arithmetic computations and string concatenations inside an SQL query, so we are able to get sums of LP variables with corresponding coefficients directly from a query. As a results, the only post processing needed is concatenation of these strings. Our grounding implementation takes comparable time to Tuffy on an MLN with comparable number of ground predicates to be generated, which is a state-of-the-art performance on this task at the moment.

To summarize, {\bf relational linear programming} works as follows: 
\begin{enumerate}
\item Specify an RLP $R$.
\item Given a LogKB, ground $R$ into an LP $L$ (Alg.~\ref{gag}).
\item Solve $L$ using any LP solver.
\end{enumerate}
In fact, this novel linear programming approach  --- as we will demonstrate later --- already ease considerably the specification of 
whole families of linear programs. However, we can do considerably better. As we will show next, we can efficiently detect and exploit symmetries 
within the induced LPs and in turn speed up solving the RLP often considerably.

%!TEX root = main.tex
\section{Exploiting Symmetries for Reducing the Dimension of LPs}\label{sec:lp}
%Loopy Belief Propagation was the first approximate inference method to receive attention from the lifted inference community. 

As we have already mentioned in the introduction, one of the features of many relational models is that 
they can produce model instances with a lot of symmetries. These symmetries in turn can be exploited to
perform inference at a ``lifted'' level, i.e., at the level of groups of variables. For probabilistic relational models, this lifted inference 
can yield dramatic speed-ups, since one reasons about the groups of indistinguishable variables as a whole, instead of treating them individually. 

Triggered by this success, we will now show that linear programming is liftable, too. 

\subsection{Detection Symmetries using Color-Passing}
%probabilistic logical models is that inference is often
%possible at the lifted (group) level. 
One way to devise a lifted inference approach is the following. One starts with a standard inference algorithm and introduces
some notion of indistinguishability among the variables in the model (instance) at hand. 
For example, we can say that two variables $X$ and $Y$ in a linear program are indistinguishable, if 
there exist a permutation of all variables, which exchanges $X$ and $Y$, yet still yields back the same model in terms of the solutions. 
Then, given a particular model instance, one detects, which variables are exchangeable in that model instance. The standard inference algorithm
is modified in such a way that it can deal with groups of indistinguishable variables as a whole, instead of treating them individually. 
This approach was for instance followed to devise a lifted version of belief propagation~\cite{singla08aaai,kersting09uai,ahmadi2013mlj}, a message-passing algorithm for approximate inference in Markov 
random fields (MRFs), which we will not briefly sketch in order to prepare the stage for lifted linear programming. In doing so, 
we will omit many details, since they are not important for developing lifted linear programming. 

\begin{algorithm}[t]
%\linesnumbered
%\SetAlgoLined
\SetKwFunction{newColor}{newColor}
\KwIn{A graph $G = (V,E)$, an initial coloring function $\lambda_0: V\cap E \rightarrow \mathbb{N}$}
\KwOut{A partition $\mathcal{U}=\{U_1,...,U_k\}$ of $V$}

Initialize $i\leftarrow 0$, $\mathcal{U}_0 = \{V\}$

\Repeat{$\mathcal{U}_{i-1} = \mathcal{U}_i$}{
  \ForEach{$v\in V$}{
    $c\leftarrow \lambda_i(v)$
    
    \ForEach{$u\in Nb_v$}{
      $c \leftarrow \left<c \cup \left( \lambda_i(u), \lambda_0(\{u,v\})  \right)\right>$  
    }
    $\lambda_{i+1}(u) \leftarrow \mathrm{hash}(c)$
 }
 $\mathcal{U}_{i+1} \leftarrow \{ \{ v\in V | \lambda_{i+1} = k \}\}$

$i \leftarrow i + 1$
}
\Return{ $\mathcal{U}_i$}\;
\caption{Color-Passing\label{alg:colpass}}
\end{algorithm}  

Belief propagation approximately computes the single-variable marginal probabilities $P(X_i)$ in an MRF encoding the joint distribution 
over the random variables $X_1$, $X_2,\ldots, X_n$.  It does so by passing messages within a graphical represention of the MRF.  

%The first approach following this ``lifted'' version of loopy BP, called Lifted First-Order Belief Propagation, was introduced by 
%Singla and Domingos~\citeyear{singla08aaai}. 
The main idea to lift belief propagation is to simulate it keeping track of which $X_i$s and clauses send identical messages. These elements of the model can then be merged into groups, whose members are indistinguishable in terms of belief propagation. After grouping elements together into 
%In turn, these groups %supervariables and superfactors 
%act as the elements of 
a potentially smaller (lifted) MRF, a modified message-passing computes the same beliefs as standard belief propagation 
on the original MRF. %In \cite{kersting09uai}, this algorithm was extended, and the lifted network construction process was given a novel interpretation: 

To identify indistinguishable elements, one first assigns colors to the elements of the MRF. %, based on the potentials for factors and on the evidence case for variables. 
Then, one performs message-passing, but replaces the belief propagation messages --- the computation of which is the most time consuming step in belief propagation, namely a sum-product operation --- by these colors. That is, instead of the sum-product update rule, one uses a less computationally intensive sort-and-hash update for colors. After every iteration, one keeps track of the partition of the network, ${\cal U} = \{U_1,...,U_k\}$, induced by nodes the same colors. Unlike standard belief propagation, this color-passing procedure is guaranteed to converge (that is, the partition will stop getting finer) in at most $|V|$ iterations. I.e., at worst, one may end up with the trivial partition. This color-passing algorithm is outlined in Alg.~\ref{alg:colpass}.

Lifted linear programming as introduced next is quite similar to lifted belief propagation. In fact, it also uses color-passing for detection the symmetries.
%symmetry detection. 
However, there are remarkable differences:
\begin{enumerate}
\item
First, lifted belief propagation applies only to approximate probabilistic inference. As it exploits redundancies within belief propagation, it will produce the same solution as belief propagation. However, examples can be found, where the true and exact solution does not exhibit symmetry, while the approximation does. In contrast, lifted linear programming is sound --- an exact solution to an LP can be recovered from any lifted solution.

\item Second, as an intermediate step lifted belief propagation generates a symmetry-compressed model. This lifted MRF which is no longer an MRF since there 
are for instance multi-edges.To accommodate for that, the message equations of lifted belief propagation are modified and, as a result, lifted message-passing cannot be done using standard belief propagation implementations. As we will show, this is not the case for lifted linear programming. The symmetry-compressed LP, or lifted LP for short, is still a LP and can be solved using any LP solver. This is a significant advantage over lifted belief propagation, as we can take full advantage of 
state-of-the-art LP solver technology.
\end{enumerate}
Both points together suggest to view lifted linear programming as  reducing the dimension of the LP. So, how do we do this?

\subsection{Equitable Partitions and Fractional Automorphisms}
\label{opt}
\label{sec:fracLP}

\begin{figure*}[t]
\centering
\includegraphics[width=0.7\textwidth]{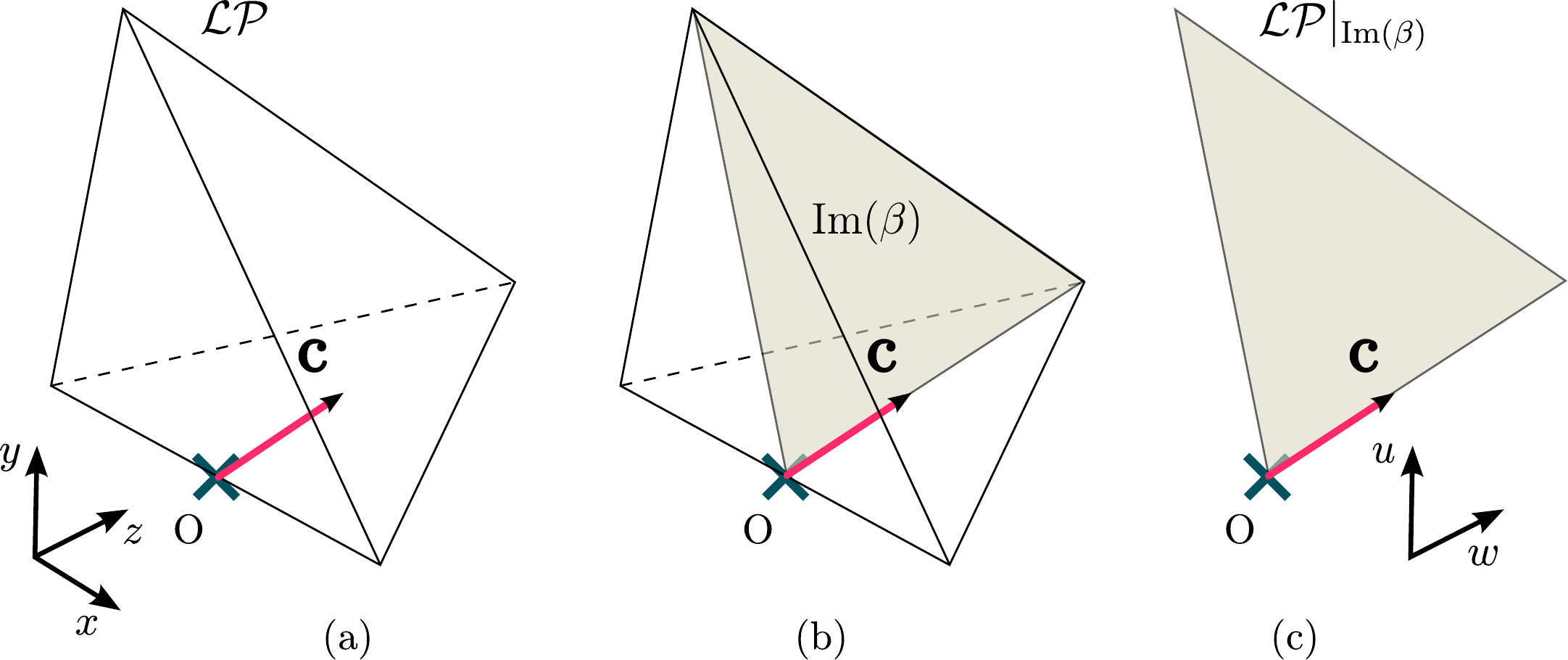}
\caption{Using symmetry to speed up linear programming: (a) the feasible region of $\lp$ and the objective vector (in pink); (b) the fixed space of $\aut(\lp)$ is identified (grey); (c) the feasible region is restricted to its intersection with the fixed space.
\label{fig:polySym}}
\end{figure*}

To develop lifted linear programming, we proceed as follows: first, we introduce the notion of {\bf equitable partitions} and show how they connect color-passing. To make use of equitable partitions in linear programs, we need to bridge the combinatorial world of partitions with the algebraic world of linear inequalities. We do so by introducing the notion of {\bf fractional automorphisms}, which serve as a matrix representation of partitions. Finally, we show that if an LP admits an equitable partition, then an optimal solution of the LP can be found in the span of the corresponding fractional automorphism. This essentially defines our lifting: by restricting the feasible region of the LP to the span of the fractional automorphism, we reduce the dimension of the LP to the rank of the fractional automorphsm (geometric intuition is given in Figure~\ref{fig:polySym}). As we will see, this results in an LP with fewer variables (to be precise, we have the same number of variables as the number of classes in the equitable partition). 

Now, let us make some necessary remarks on the nature of the partition $\cal U$ returned by Alg.~\ref{alg:colpass}. Suppose for now, $\lambda_0(e)=1$ for all $e\in E$, i.e. the graph is only vertex-colored. Observe that according to line $8$, nodes $v$ and $v^\prime$ receive different colors in the $i$'th iteration if the multisets of the colors of their neighbors are different. That is, in order to be distinguished by Alg.~\ref{alg:colpass}, $v$ and $v^\prime$ must have a different number of neighbors of the same color at some iteration $i$. 

Consequently, as the algorithm terminates when the partition no longer refines, we conclude that the following holds: Alg.~\ref{alg:colpass} partitions a graph $G$ in such a way that every two nodes nodes in the same class have the same number of neighbors from every other class. More formally, for each pair classes $U_i, U_j$ and every two nodes $v,v^\prime \in U_i$, we have 
\[|\nb(v)\cap U_j| = |\nb(v^\prime) \cap U_j|\;.\]
We call any partition with the above property {\bf equitable}. For the edge-colored case, we have a slightly more complicated definition: $\cal U$ is {\bf equitable}, whenever it holds that for each pair classes $U_i, U_j$, every edge color $c$, and every two nodes $v,v^\prime \in U_i$
\[|\{w \in U_j\ |\ \lambda_0(\{v,w\}) = c \}| = |\{w^\prime \in U_j\ |\ \lambda_0(\{v^\prime,w^\prime\}) = c \}|\;.\]
In other words, we say that every two nodes in a class have the same number of edges of the same color going into every other class.

 In fact, it can be shown that Alg.~\ref{alg:colpass} computes the coarsest such partition of a graph, and 
%In fact, 
the relationship between equitable partitions and color-passing has been well-known in graph theory~\cite{Ramana1994}.
% {\todo theorems of the birkhoff type}. 
%and one can show that the colors computed by color-passing in every iteration correspond to the so-called iterated degree (or color for the case of colored graphs) sequences.

Observe that like MRFs, which have variables and factors, linear programs are bipartite objects as well: they consists of variables and constraints. As such, they are more naturally represented by bipartite graphs, hence we will now narrow down our discussion to bipartite graphs. We say that a colored graph $G = (V,E,\lambda)$ is {\bf biparite}, if the vertex set consists of two subsets, $V=A\cup B$, such that every edge in $E$ has one end-point in $A$ and the other in $B$. The notion of equitable partitions apply naturally to the bipartite setting -- we will use the notation ${\cal U} = \{P_1,\ldots,P_p, Q_1,\ldots,Q_q\}$ to indicate that the classes over the subset $A$ are disjoint from the classes over the subset $B$. I.e., no pair $v\in A, v^\prime \in B$ can be in the same equivalence class.   

With this in mind, we introduce fractional automorphisms. 
Note that our definition is slightly modified from the original one \cite{Ramana1994}, in order to accomodate the bipartite setting.
\begin{definition}
Let $\bf M$ be an $m\times n$ real matrix, such as the (colored, weighted) adjacency matrix of a bipartite graph. A fractional automorphism of $\bf M$ is a pair of doubly stochastic (meaning that the entries are non-negative, and every row and column sum to one) matrices $\+X_P, \+X_Q$ such that \[\+M\+X_P = \+X_Q\+M.\]
\end{definition}
%A surprising and very useful fact, which plays a central role in this work, is that 
The following theorem establishes the correspondence between equitable partitions and fractional automorphisms.
\begin{theorem}[\cite{Ramana1994,grohe13arxiv}]
\label{thm:fracauto}
Let $G$ be a bipartite graph. Then:
\begin{itemize}
\item[i)] if ${\cal U} = \{P_1,\ldots,P_p, Q_1,\ldots,Q_q\}$ is an equitable partition of $G$, then the matrices $\+X_P, \+X_Q$, having entries 
\begin{align}
\label{eq:flatmatrix}
\+(X_P)_{ij} =
\begin{cases}
 1/|P| & \text{ if both vertices } i,\ j \text{ are in the same } P \in {\cal U},\\
 0 & \text{ otherwise,}
\end{cases}\nonumber\\
\+(X_Q)_{ij} =
\begin{cases}
 1/|Q| & \text{ if both vertices } i,\ j \text{ are in the same } Q \in {\cal U},\\
 0 & \text{ otherwise,}
\end{cases}
\end{align}
is a fractional automorphism of the adjacency matrix of $G$.  
\item[ii)] conversely, let $\+X_P,\+X_Q$ be a fractional automorphism of the (colored, weighted) adjacency matrix of the bipartite $G$ with edge set $V=A\cup B$. Then the partition ${\cal U}$, where vertices $i,j \in A$ belong to the same $P\in {\cal U}$ if and only if at least one of $(\+X_P)_{ij}$ and $(\+X_P)_{ji}$ is greater than $0$, respectively $i,j\in B$ belong to a class $Q$ if $(\+X_Q)_{ij}$ or $(\+X_Q)_{ji}$ is greater than $0$,  is an equitable partition of $G$.
\end{itemize}
\end{theorem}
%\begin{proof}
%See ramana or grohe
%\end{proof}
In the following, part $i)$ of the above Theorem will be of particular interest to us. We will shortly show how we can use color-passing to construct equitable partitions of linear programs. Encoding these equitable partitions as fractional automorphisms will provides us with insights into the geometrical aspects of lifting and will be an essential tool for proving its soundness.

Note that any graph partition (equitable or not) can be turned into a doubly stochastic matrix using~\eqref{eq:flatmatrix}. However, keep in mind that the resulting matrix will not be a fractional automorphism unless the partition is equitable. In any case, partition matrices have a useful property that will later on  allow us to reduce the number of constraints and variables of a linear program. Namely, 
%the A practical characterization of the relationship between equitable partitions and fractional automorphisms comes from C.D. Godsil. That is,
\begin{proposition}[\cite{god97}]\label{prop:godsil}
Let $\+X$ be the doubly stochastic matrix of some partition ${\cal U}$ according to \eqref{eq:flatmatrix}. Then $\+X = \widetilde{\+B}\widetilde{\+B}^T$, with  
\begin{equation}
\label{eq:charmatrix}
{\widetilde{\bf B}}_{im} = 
 \begin{cases}
   \frac{1}{\sqrt{|U_m|}} & \text{ if vertex } i \text{ belongs to part } U_m,\\
   0       & \text{ otherwise. }
  \end{cases}
\end{equation}
\end{proposition}
%\begin{proof}
%See~\cite{god97}.
%\end{proof}

\subsection{Fractional Automorphisms of Linear Programs}

% The first connection between linear programs and fractional automorphisms was established in \cite{Rubalcaba2005} while investigating graph properties which are shared between fractionally isomorphic graphs. We will now build upon Rubalcaba's argument to show that LPs, like MRFs, can be lifted using Alg.~\ref{alg:colpass}.
As we already mentioned, we are going to apply equitable partitions through fractional automorphisms to reduce the size of linear programs. Hence, the obvious question presents itself: what is an equitable partition (resp. fractional automorphism) of linear program?

In order to answer the question, we need a graphical representation of $L=(\+A,\+b,\+c)$, called the {\bf coefficient graph} of $L$, $G_L$. To construct $G_L$, we add a vertex to $G_L$ for every of the $m$ constraints $n$ variables of $L$. Then, we connect a constraint vertex $i$ and variable vertex $j$ \emph{if and only if} ${\bf A}_{ij}\neq 0$. Furthermore, we assign colors to the edges $\{i,j\}$ in such a way that $\colorv(\{i,j\}) = \colorv(\{u,v\}) \iff {\bf A}_{ij} = {\bf A}_{uv}$. Finally, to ensure that  $\bf c$ and $\bf b$ are preserved by any automorphism we find, we color the vertices in a similar manner, i.e., for row vertices $i,j$ $\colorv(i) = \colorv(j) \iff \+b_i = \+b_j$ and $color(u) = color(v) \iff {\bf c}_u = {\bf c}_v$ for column vertices. We must also choose the colors in a way that no pair of row and column vertices share the same color; this is always possible. %This is summarized in Algorithm~\ref{}.
 
\begin{figure*}[t]
\centering
\includegraphics[width=0.5\textwidth]{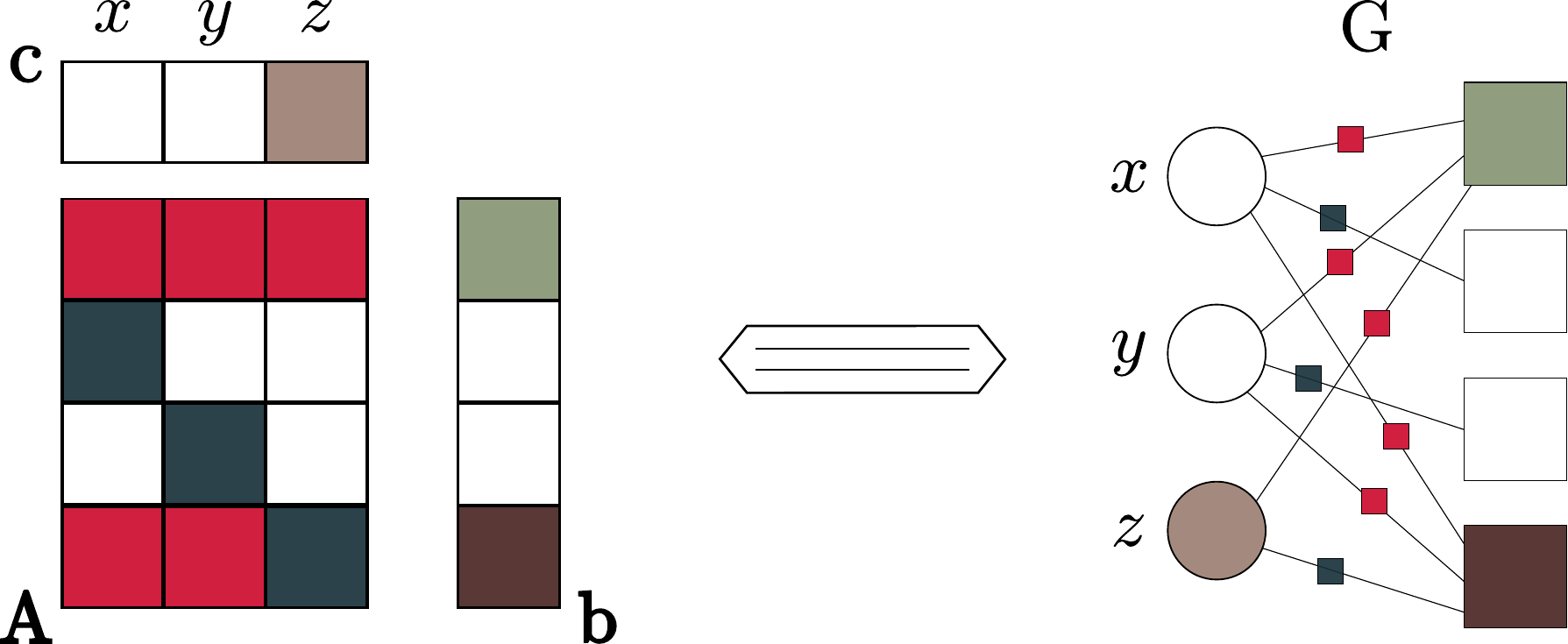}
\caption{Construction of the coefficient graph $G_L$ of $L^0$. On the left-hand side, the coloring
of the LP is shown. This turns into the colored coefficient graph shown on the right-hand side.
\label{fig:symDet}}
\end{figure*}

To illustrate this, consider the following toy LP: 
\begin{lstlisting}[language=ampl,frame=none,basicstyle=\footnotesize\ttfamily]
var p/1;

maximize: sum_{gadget(X)} p(X);

subject to: sum{widget(X)}p(X) + sum{gadget(X)} p(X) <= 1;
subject to {widget(X)}: widget(X) <= 0;
subject to: sum{widget(X)}p(X) - sum{gadget(X)} p(X) <= -1;
\end{lstlisting}
with knowledge base LogKB (recall that logical atoms are assumed to evaluate to $0$ and $1$ within an RLP):
\begin{lstlisting}[language=ampl,frame=none,basicstyle=\footnotesize\ttfamily,numbers=none,backgroundcolor=\color{blue!10}]
widget(x).
widget(y).
gadget(z).
\end{lstlisting}
If we ground this linear program and convert it to dual form (as in \eqref{eq:dualform}), we obtain the following linear program $L^0 = (\+A,\+b,\+c)$  
\begin{align*}
    \operatorname*{minimize}_{[x,y,z]^T \in \mathbb{R}^3}\quad  &\; 0x + 0y + 1z\\ 
    \text{subject to}\quad & \begin{bmatrix}
       1 & 1 & 1           \\[0.3em]
       -1 & 0 & 0            \\[0.3em]
       0 & -1 & 0            \\[0.3em]
       1 & 1 & -1            \\[0.3em]
     \end{bmatrix} 
      \begin{bmatrix}
       x\\
       y\\
       z\\
     \end{bmatrix} \leq 
           \begin{bmatrix}
       1\\
       0\\
       0\\
       -1\\
     \end{bmatrix}\;,
 \end{align*}
where for brevity we have substituted ${\tt p}(x), {\tt p}(y), {\tt p}(z)$ by $x,y,z$ respectively. The coefficient graph
of $L^0$ is shown in Fig.~\ref{fig:symDet}.

We call an {\bf equitable partition of a linear program} $L$ the equitable partition of the graph $G_L$\footnote{using the notion of equitable partitions of bipartite colored graphs from the previous section.}.  
Suppose now we compute an equitable partition ${\cal U} = \{P_1,\ldots,P_p, Q_1,\ldots,Q_q\}$ of $G_L$ using Algorithm~\ref{alg:colpass} and compute the corresponding fractinal automorphism $(\+X_P, \+X_Q)$ as in \eqref{eq:flatmatrix}. Observer that $(\+X_P, \+X_Q)$ will have the following properties: 
\begin{itemize}
\item[i)] due to Theorem~\ref{thm:fracauto}, we have $\+X_Q\+A=\+A\+X_P$;
\item[ii)] by our choice of initial colors of $G_L$, the partition ${\cal U}$ will never group together variable vertices $i,j$ with $\+c_i \neq \+c_j$, nor will it group constraint vertices $i,j$ with $\+b_i \neq \+b_j$. By \eqref{eq:flatmatrix}, this implies \[\+c^T\+X_P = \+c^T\]  and \[\+X_Q\+b = \+b\;.\]   
\end{itemize}

This yields the definition of a {\bf fractional automorphism of linear programs} -- we call a pair of doubly stochastic matrices $(\+X_P, \+X_Q)$ a {\bf fractional automorphism of the linear program } $L$ if it satisfies properties $i)$ and $ii) $  as above. 
\subsection{Lifted Linear Programming}

With this at hand, we are ready for the main part of our argument. We split the argument in two parts: first, we will show that if an LP $L$ has an optimal solution $\+x^*$, then it also has a solution $\+X_P\+x^*$. That is, if we add the constraint $\+x \in \operatorname{span}(\+X_P)$ to the linear program, we will not cut away the optimum. The second claim is that $\+x \in \operatorname{span}(\+X_P)$ can be realized by a projection of the LP into a lower dimensional space. So, we can actually project to a low-dimensional space, solve the LP there, and then recover the high-dimensional solution via simple matrix multiplication. We recall that his idea is illustrated in Figure~\ref{fig:polySym}.

Let us now state the main result. 
\begin{theorem}
Let $L=(\+A,\+b,\+c)$ be a linear program and $(\+X_P,\+X_Q)$ be a fractional automorphism of $L$. Then, it holds that if $\+x$ is a feasible in $L$, then $\+X_P\+x$ is feasible as well and both have the same objective value. As a consequence, if $\+x^*$ is an optimal solution, $\+X_P\+x^*$ is optimal as well.
\end{theorem} 

\begin{proof}
Let $\+x$ be feasible in $L = (\+A,\+b,\+c)$, i.e. $\+A\+x \leq \+b$. Observe that left multiplication of the system by a doubly stochastic matrix preserves the direction of inequalities. More precisely \[\+A\+x \leq \+b \Rightarrow \+S\+A\+x \leq \+S\+b\;,\] for any doubly stochastic $\+S$\footnote{actually, this holds for any positive matrix $\+S$, since $\+S(\+A\+x)_i = \sum_j \+S_{ij}(\+A\+x)_j \leq \sum_j \+S_{ij}\+b_j = (\+S\+b)$, since $\+S$ is positive and $(\+A\+x)_j \leq \+b_j$ by assumption. }.
 So now, we left-multiply the system by $\+X_Q$:
 \[\+A\+x \leq \+b \Rightarrow \+X_Q\+A\+x \leq \+X_Q\+b = \+A\+X_P\+x \leq \+b ,\]
 since $\+X_Q\+A = \+A\+X_P$ and $\+X_Q\+b = \+b$. This proves the first part of our Theorem. Finally, observe that $\+c^T(\+X_P\+x) = \+c^T\+x$ as $\+c^T\+X_P = \+c^T$.  
\end{proof}

We have thus shown that if we add the constraint $\+x \in \operatorname{span}(\+X_P)$ to $L$, we can still find a solution of the same quality as in the original program. How does this help to reduce dimensionality? To answer, we observe that the constraint $\+x \in \operatorname{span}(\+X_P)$ can be implemented implicitly, through reparametrization. That is, instead of adding it to $L=(\+A,\+b,\+c)$ explicitly, we take the LP $L^\prime = (\+A\+X_P,\+b,\+X_P^T\+c)$. Now, recall that $X_P$ was generated by an equitable partition, and it can be factorized as $X_P = \widetilde{\+B}_P\widetilde{\+B}_P^T$ where $\widetilde{\+B}_P$ is the normalized incidence matrix of $\{P_1,\ldots,P_p\}\subset{\cal U}$ as in \eqref{eq:charmatrix}.

Note that
the span of $\+X_P = \widetilde{\bf B}_P\widetilde{\bf B}_P^T$ is equivalent (in the vector space isomorphism sense) to the column space of $\widetilde{\bf B}_P$. That is, every $\+x\in\mathbb{R}^n$, $\+x \in \operatorname{span}(\+X_P)$ can be expressed as $\+x = \widetilde{\bf B}_P\+y$ some $y\in \mathbb{R}^p$ and conversely, $\widetilde{\bf B}_P\+y \in\operatorname{span}(\+X_P)$ for all $\+y\in\mathbb{R}^p$.Hence, we can replace $L^\prime = (\+A\+X_P,\+b,\+X_P^T\+c)$ with the equivalent $L^{\prime\prime} = (\+A\widetilde{\bf B}_P,\+b,\widetilde{\bf B}^T_P\+c)$ . Since this is now a problem in $p \leq n$ variables, i.e., of reduced dimension, 
%{\todo this can be phrased better}) 
a speed-up of solving the original LP is possible. Finally, by the above, if $\+y^*$ is an optimal solution of $L^{\prime \prime}$, $\widetilde{\bf B}_P\+y$ is an optimum solution of $L$.

\begin{algorithm}[t]
%\linesnumbered
%\SetAlgoLined
\SetKwFunction{newColor}{newColor}
\KwIn{An inequality-constrained LP, $L=(\mathbf{A}, \mathbf{b}, \mathbf{c})$}
\KwOut{$\mathbf{x}^{*} = \argmin_{\{\mathbf{x}|\mathbf{Ax} \leq \mathbf{b}\}}{\mathbf{c}^T\mathbf{x}}$}
Construct the coefficient graph $G_L$\;
Lift $G_L$ using color-passing, see Alg.~\ref{alg:colpass}.\;
Read off the characteristic matrix $\widetilde{\mathbf{B}}_P$\;
Obtain the solution $\mathbf{y}$ of the LP $(\mathbf{A}\widetilde{\mathbf{B}}_P,\mathbf{b},\widetilde{\mathbf{B}}_P^T\mathbf{c})$ using any standard LP solver\;
\Return{$\mathbf{x}^{*} = \widetilde{\mathbf{B}}_P\mathbf{y}$}\;
\caption{Lifted Linear Programming\label{alg:duallift}}
\end{algorithm} 

%So far, we have reduced the number of LP variables using the equivalence classes of the equitable partition. 

Overall, this yields the {\bf lifted linear programming} approach as summarized in Alg.~\ref{alg:duallift}. Given an LP,
first construct the coefficient graph of the LP (line 1). Then, lift the coefficent graph (line 2)
and read off the characteristic matrix (line 3). Solve the lifted LP (line 4) and ``unlift'' the
lifted solution to a solution of the original LP (line 5). 

Applying this lifted linear programming to LPs induced by RLPs, we can rephrase  
{\bf relational linear programming} as follows: 
\begin{enumerate}
\item Specify an RLP $R$.
\item Given a LogKB, ground $R$ into an LP $L$ (Alg.~\ref{gag}).
\item Solve $L$ using lifted linear programming (Alg.~\ref{alg:duallift}).
\end{enumerate}

Before illustrating relational linear programming, we would like to note that this method of constructing fractional automorphisms works for any equitable partition, not only the one resulting from running color-passing. For example, an equitable partition of a graph can be constructed out of its automorphism group, by making two vertices equivalent whenever there exists a graph automorphism that maps one to the other. The resulting partition is called the orbit partition of a graph, see e.g.~\cite{godsil01}. 
Applying this partitioning method to coefficient graphs of linear programs and using the corresponding fractional automorphism is equivalent to previous theoretical well-known results in solving linear programs under symmetry,
see e.g.~\cite{Boedi13} and references in there. However, there are two major benefits for using the color-passing partition instead of the orbit partition:
% for lifting/compressing LPs:
\begin{itemize} 
\item The color-passing partition is at least as coarse as the orbit 
partition, see e.g.~\cite{godsil01}. To illustrate this, consider the so-called Frucht graph  as shown in Fig.~\ref{fig:frucht}. Suppose we turn this graph into a linear program by introducing constraint nodes along the edges and coloring everything with the same color. The Frucht graph has two extreme properties with respect to equitable partitions: 1) it is asymmetric, meaning that the orbit partition is trivial -- one vertex per equivalence class; 2) it is regular (every vertex has degree 3); as one can easily verify, in this case the coarsest equitable partition consists of a single class! 
\begin{figure*}[t]
\centering
\includegraphics[trim = 0cm 0cm 11cm 0cm, clip=true,width=0.6\textwidth]{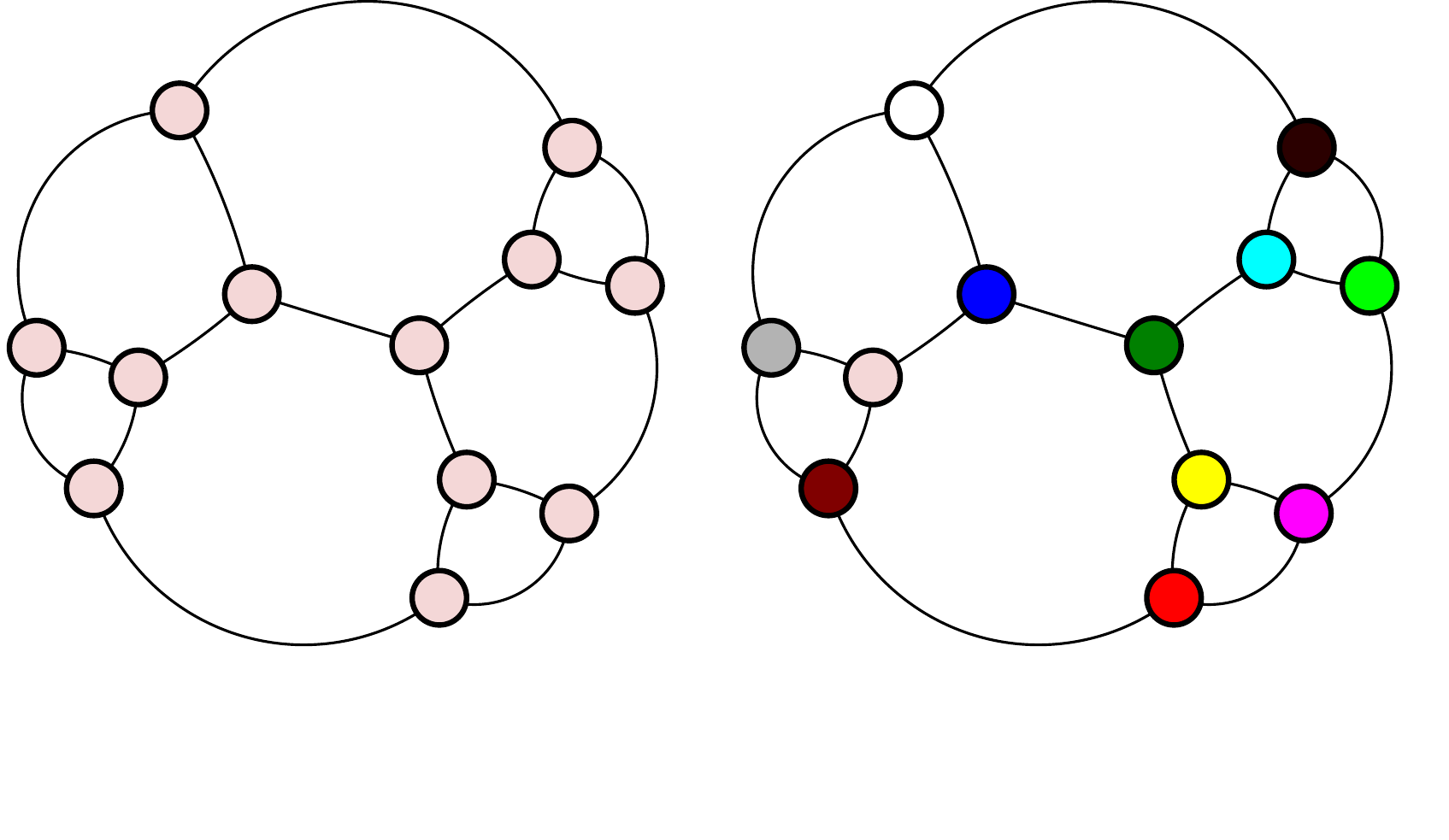}
\caption{The Frucht graph with 12 nodes. The colors indicate the resulting node partitions using color-passing (the coarsest equitable partition, eft) and using automorphisms (the orbit partition, right).}
\label{fig:frucht}
\end{figure*} 
Due to these two properties, in the case of the Frucht graph the orbit partition yields no compression, whereas the coarsest equitable resp.~color-passing partition produces an LP with a single variable.
\item The color-passing partition can be computed in quasi-linear time~\cite{berbongro13}, yet current tools for orbit partition enumeration have significantly worse running times. Thus, by using color-passing we achieve strict gains in both compression and efficiency compared to using 
orbits. 
\end{itemize}
Let us now illustrate relational linear programming using several AI tasks.

%!TEX root = main.tex

\section{Illustrations of Relational Linear Programming}
\label{exp}
Our intention here is to investigate the viability of the ideas and concepts of relational linear programming
through the following questions:
\begin{description}
\item[(Q1)] Can important AI tasks be encoded in a concise and readable relational way using RLPs?
\item[(Q2)] Are there (R)LPs that can be solved more efficiently using lifting? 
\item[(Q3)] Does relational linear programming enable a programming approach to AI tasks
facilitating the construction of more sophisticated models from simpler ones by adding rules?
\end{description}
If so, relational linear programming has the potentially to make linear models faster to write and easier to understand,
reduce the development time and cost to encourage experimentation, and in turn 
reduce the level of expertise necessary to build AI applications. Consequently, our primary
focus is not to achieve the best performance by using advanced models. Instead we will  focus
on basic models. 

We have implemented a prototype system of relational linear programming, and 
illustrate the relational modeling of several AI tasks:  computing the value function of Markov decision processes, 
performing MAP LP inference in Markov logic networks and performing
collective transductive classification using LP support vector machines.

\subsection{Lifted Linear Programming for Solving Markov Decision Processes}
\label{sec:mdp}
Our first application for illustrating relational linear programing is the computation of the value function of Markov Decision Problems (MDPs). 
The LP formulation of this task is as follows \cite{Littman95}:
\begin{align}\label{eq:mdp}
    \operatorname*{maximize}\nolimits_{\mathbf{v}} & \mathbf{1}^T\mathbf{v}, \nonumber\\
    \ \text{subject to} & \ v_i \leq c_i^k + \gamma\sum\nolimits_{j \in \Omega_S}p_{ij}^k v_j\;,
\end{align}
where $v_i$ is the value of state $i$, $c_i^k$ is the reward that the agent receives when carrying out action $k$, and $p_{ij}^k$ is the probability of transferring from state $i$ to state $j$ by taking action $k$.  $\gamma$ is a discounting factor. The corresponding RLP is given in Fig.~\ref{rlp:mdp}. Since it abstract away
the states and rewards --- they are defined in the LogKB --- it extracts the essence of computing value functions of MDPs. Given a LogKB, a ground LP is
automatically created Instead of coding the  LP by hand for each problem instance again and again as in vanilla linear programming. 
This answers question {\bf (Q1)} affirmatively.

 \begin{figure}[t]
\begin{lstlisting}[language=ampl]
var value/1;             #value function to be determined by the LP solver

maximize: sum{reward(S,_)} value(S);          #best values for all states

#encoding of discounted Bellman optimality as in inequality (5)
subject to {transProb(S,T,_)}: value(S) <= reward(S,A) +    
	gamma*sum{transProb(S,T,A)} transProb(S,T,A)*value(T);
\end{lstlisting}
\caption{An RLP for computing the value function {\tt value/1} of a Markov decision process. There is a finite set of states and actions, and the agent receives a reward {\tt reward(S,A)} for performing and action {\tt A} in state {\tt S}, specified in a LogKB.\label{rlp:mdp}}
\end{figure}

The MDP instance that we used is the well-known Gridworld (see e.g.\ \cite{suttonbarto98}). The gridworld problem consists of an agent navigating within a grid of $n\times n$ states. 
Every state has an associated reward $R(s)$. Typically there is one or several states with high rewards, considered the goals, whereas the other states have
zero or negative associated rewards.  We considered an instance of gridworld with a single goal state in the upper-right corner with a reward of $100$.
The reward of all other states was set to $-1$. 
The scheme for the LogKB looked like this:
\begin{lstlisting}[language=ampl,frame=none,basicstyle=\footnotesize\ttfamily,numbers=none,backgroundcolor=\color{blue!10}]
reward(state(n-1,n),right)=100.
reward(state(n,n-1),up)=100.
reward(state(X,Y),_)=-1 :- X>0, X<n-1, Y>0, Y<n-1.
\end{lstlisting}
As can be seen in Fig.~\ref{fig:ExpsMDP}(a), this example can be compiled to
about half the original size. Fig.~\ref{fig:ExpsMDP}(b) shows that already this compression leads to improved running time.
We now introduce additional symmetries by putting a goal in every corner of the grid. As one might expect this additional
symmetry gives more room for compression, which further improves efficiency as reflected in Figs.~\ref{fig:ExpsMDP}(c) and ~\ref{fig:ExpsMDP}(d).

\begin{figure*}[t]
\centering
\subfloat[Ground vs. lifted variables on a basic gridworld MDP.]{
\includegraphics[width=0.3\textwidth]{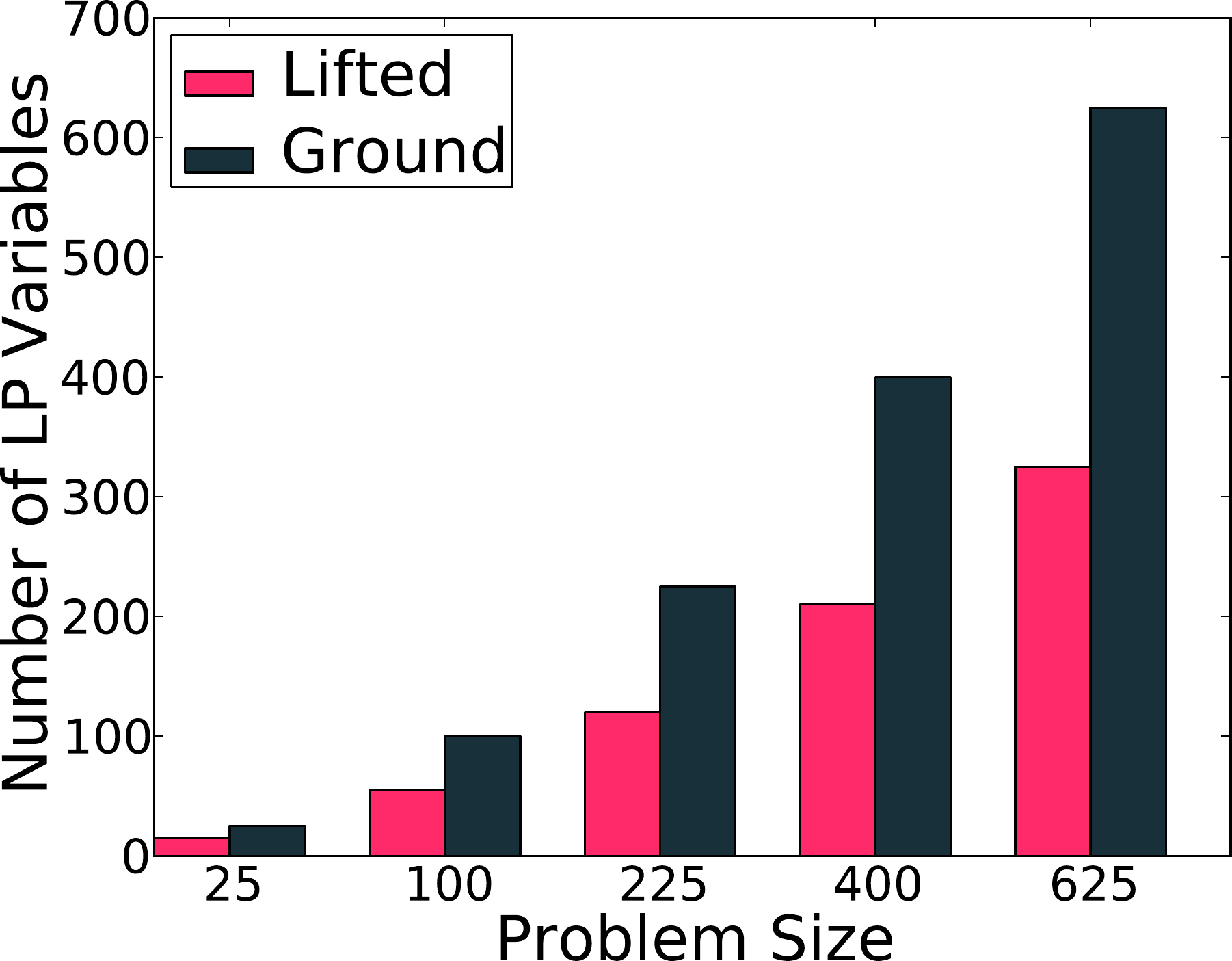}
}
\hfill
\subfloat[Measured times on a basic gridworld MDP.]{
\includegraphics[width=0.3\textwidth]{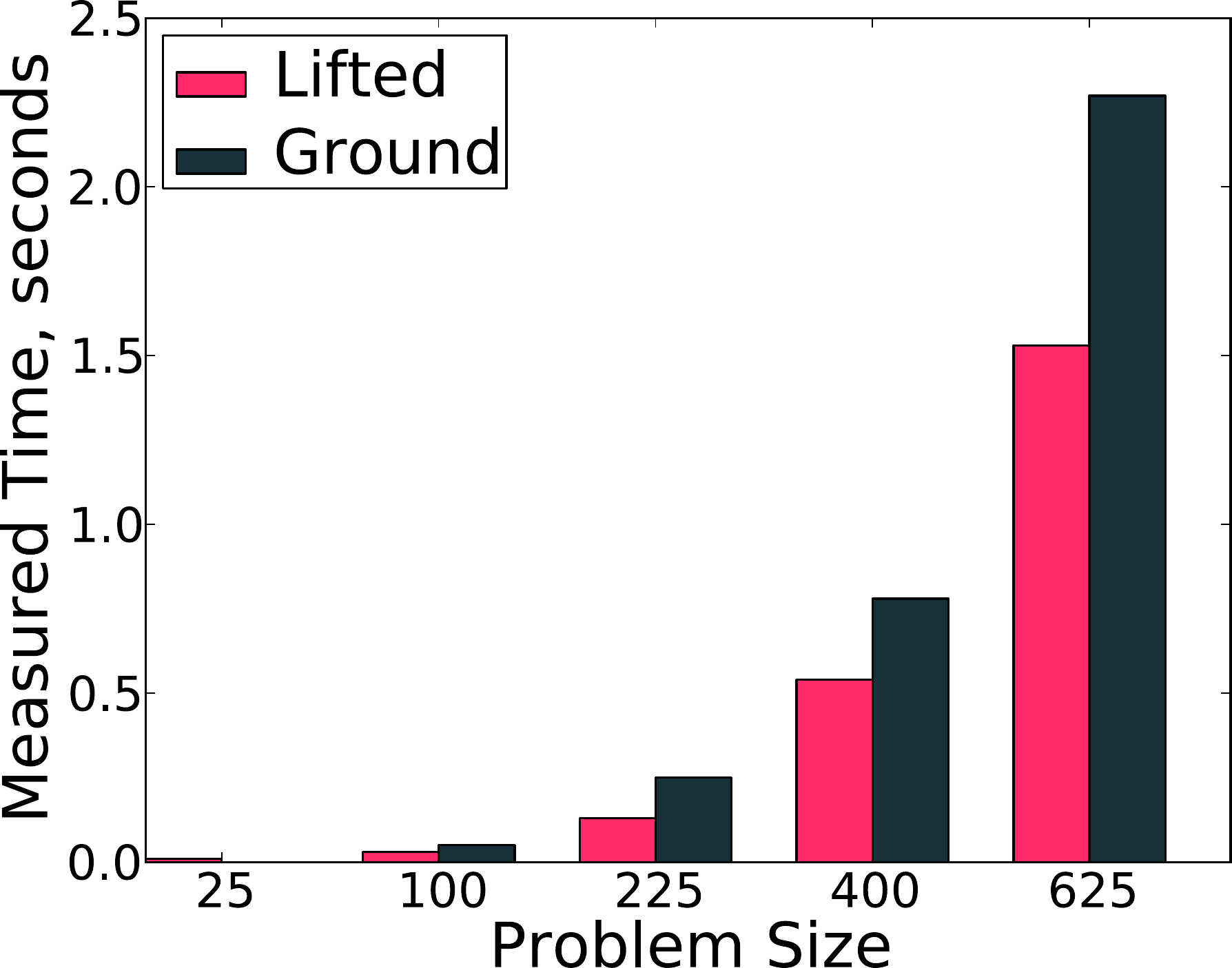}
}
\hfill
\subfloat[Variables on a gridworld with additional symmetry.]{
\includegraphics[width=0.3\textwidth]{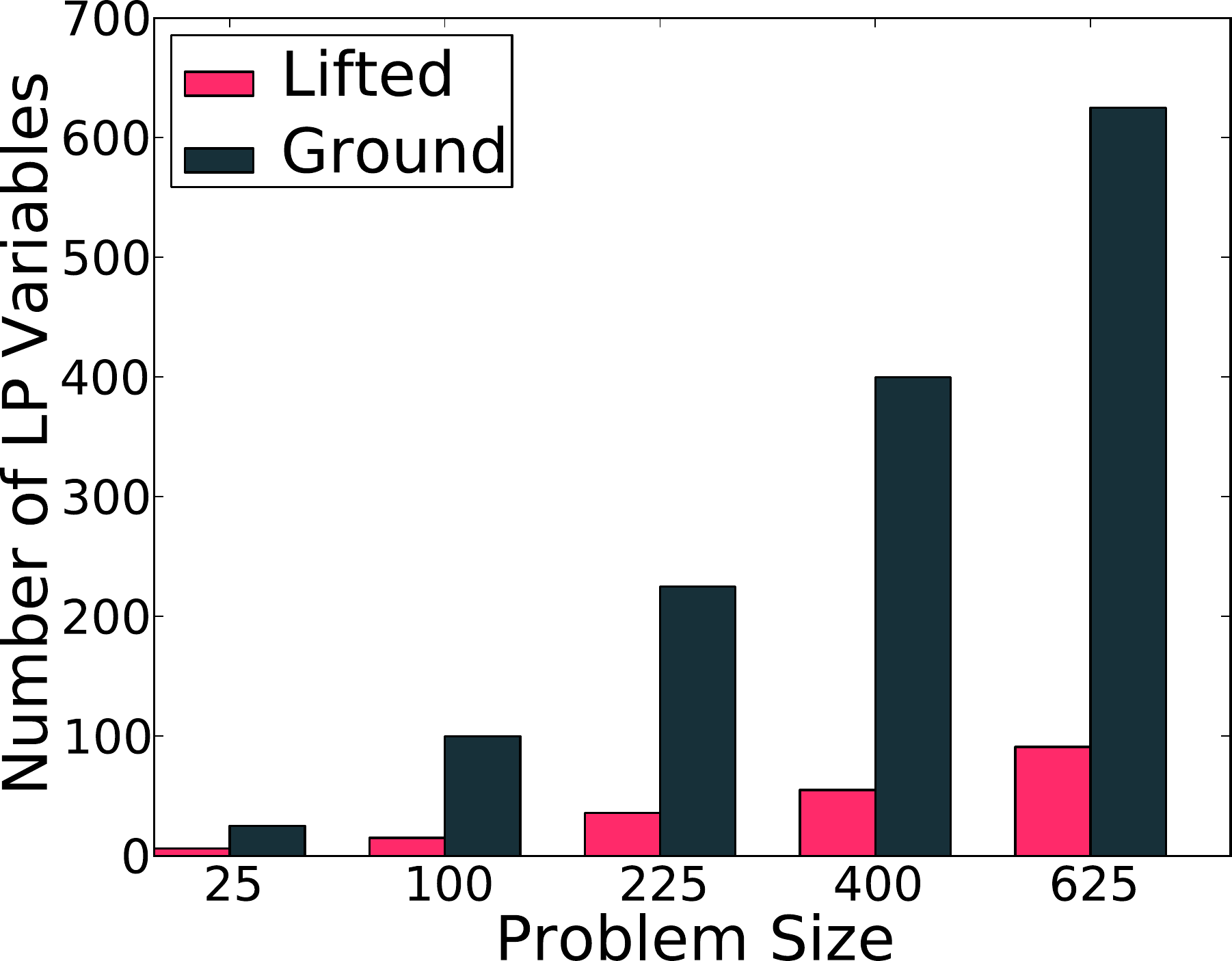}

}
\hfill
\subfloat[Measured times on a gridworld with additional symmetry.]{
\includegraphics[width=0.3\textwidth]{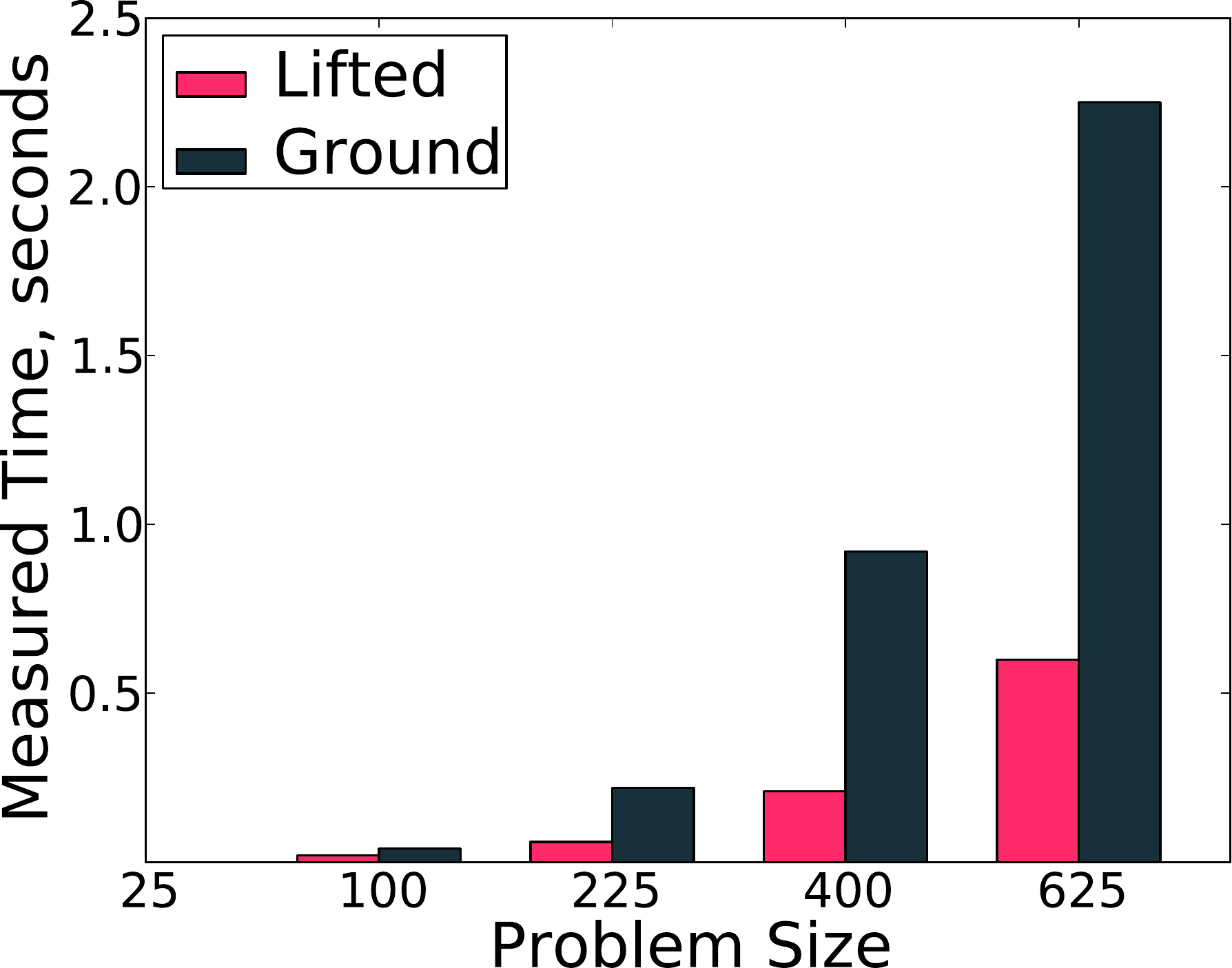}
}
\hfill
\subfloat[Variables on a gridworld with additional symmetry in sparse form.]{
\includegraphics[width=0.3\textwidth]{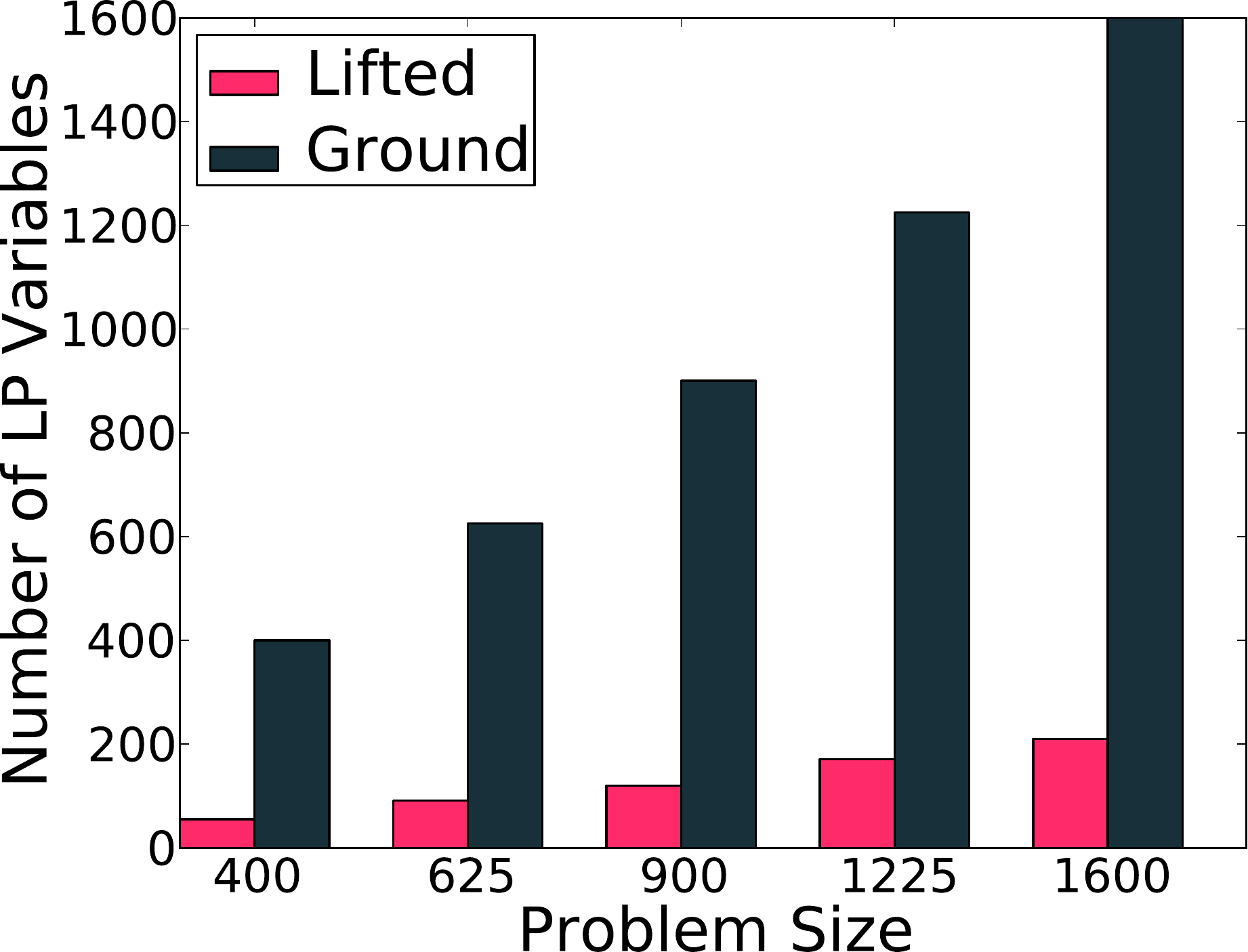}
}
\hfill
\subfloat[Measured times on a gridworld with additional symmetry in sparse form.]{
\includegraphics[width=0.3\textwidth]{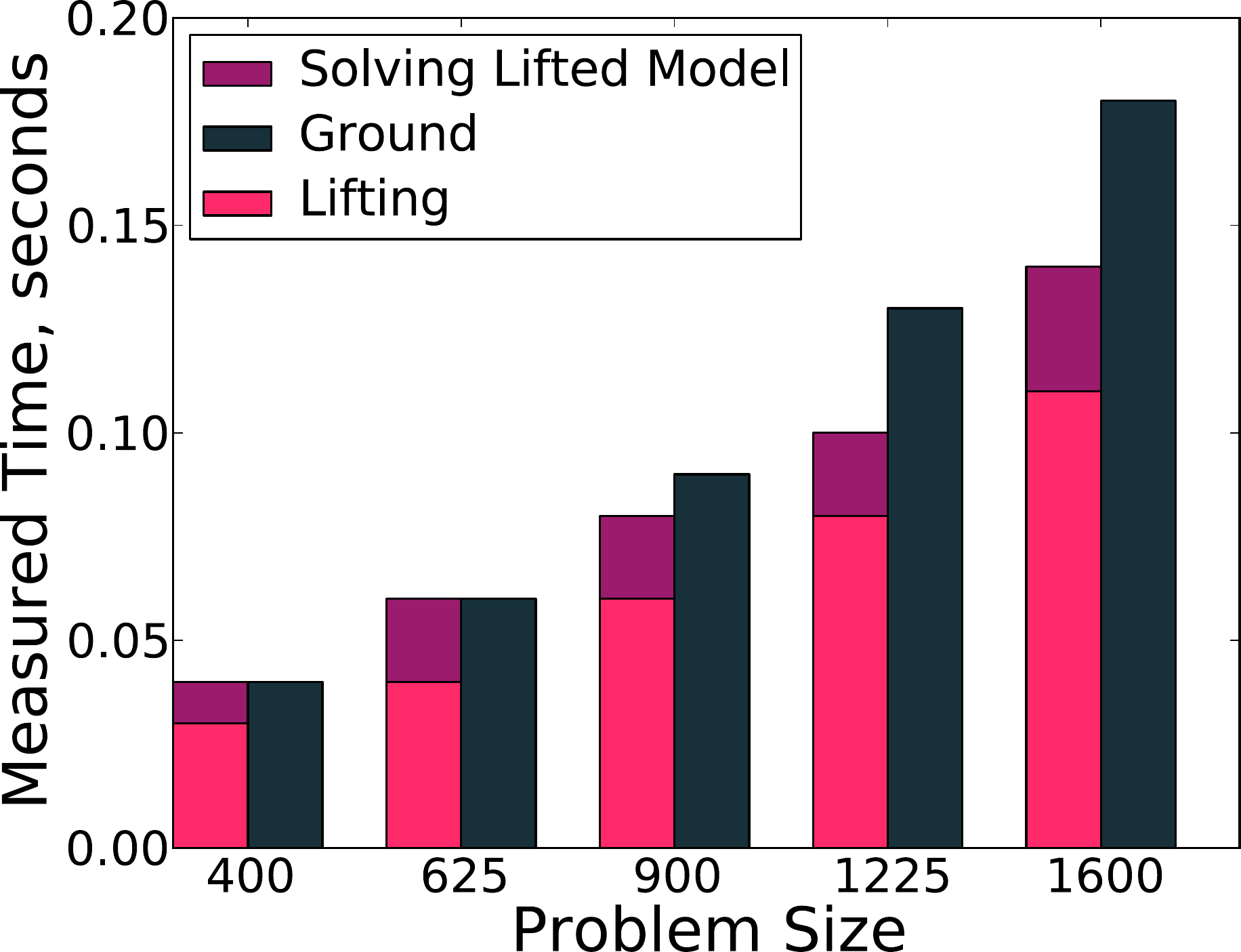}
}
\caption{\label{fig:ExpsMDP} Experimental results of relational linear programming for solving Markov decision processes.}
\end{figure*}

%The two experiments presented so far affirmatively answer question {\bf (Q1)}.
However, the examples that we have considered so far are quite sparse in their structure. 
Thus, one might wonder whether the demonstrated benefit is achieved only because we are solving sparse problem in dense form.
To address this we convert the MDP problem to a sparse representation for our further experiments. 
We scaled the number of states up to $1600$ and as one can see in Fig.~\ref{fig:ExpsMDP}(e) and (f) lifting still results in an improvement of size as well as running time.
Therefore, we can conclude that lifting an LP is beneficial regardless of whether the problem is sparse or dense, thus one might view symmetry as a dimension orthogonal to sparsity. Furthermore, in Fig.~\ref{fig:ExpsMDP}(f)
we break down the measured total time for solving the LP into the time spent on lifting and solving respectively.
This presentation exposes the fact that the time for lifting dominates the overall computation time. Clearly, if lifting was carried out in every iteration
(CVXOPT took on average around 10 iterations on these problems) the approach would not have been competitive to simply solving on the ground level.
This justifies that the loss of potential lifting we had to accept in order to not carry out the lifting in every iteration indeed pays off
{\bf (Q2)}. 
Remarkably, these results follow closely what has been achieved with MDP-specific symmetry-finding and model minimization approaches~\cite{ravi08icml,ravi01tech,Dean97}.

\subsection{Programming MAP LP Inference in Markov Logic Networks}
\begin{figure}[t]
\begin{lstlisting}[language=ampl]
var m/2;  #single node, pairwise, and
var m/4;  #triplewise probabilities
var m/6;  #of configurations to be determined by the solver
#value of the MAP assignment
innerProd = sum{w(P, V)} w(P, V) * m(P, V) + 
	sum{w(P1, P2, V1, V2)} w(P1, P2, V1, V2) * m(P1, P2, V1, V2) + 
	sum{w(P1, P2, P3, V1, V2, V3)} w(P1, P2, P3, V1, V2, V3) * 
	m(P1, P2, P3, V1, V2, V3);
atomMarg(P) = sum {w(P, V)} m(P, V);                #single node marginals
#single node marginal computed from pairwise marginals
clauseMargl1(P1, P2, V1) = sum{w(P2, V2)} m(P1, P2, V1, V2);
...
#single node marginal computed from triplewise marginales 
clauseMarg1(P1, P2, P3, V1) = sum{w(P3, V3), w(P2, V2)} 
	m(P1, P2, P3, V1, V2, V3);
...
maximise: innerProd;               #find MAP assignment with largest value  

subject to {w(P, _)}: atomMarg(P) = 1;           #normalization constraint 
#pairwise consistency constraints
subject to {w(P1, P2, V1, _)}: m(P1, V1) - clauseMarg1(P1, P2, V1) = 0;
...
#triplewise consistency constraints
subject to {w(P1, P2, P3, V1, _, _)}: 
         m(P1, V1) - clauseMarg1(P1, P2, P3, V1) = 0;
...
\end{lstlisting}
\caption{RLP encoding the MAP-LP for triplewise MLNs as shown in~\eqref{triple}. The last two constraints as well as the last two aggregates have symmetric copies that have been omitted (this redundancy is necessary, since logic predicates are not symmetric). \label{lp:map}}
\end{figure}

MLNs, see ~\cite{richardson2006markov} for more details, are a prominent model in statistical relational learning (SRL).  We here focus on MAP (maximum a posteriori) inference where we want to find a most likely joint assignment to all the random variables.  More precisely, 
an MLN induces a Markov random field (MRF) with a node for each ground atom and a clique for every ground formula. A common
approach to approximate MAP inference in MRFs is based on LP, see e.g.~\cite{wainwright2008graphical} for a general overview. 
Actually, there is a hierarchy of LP formulation for MAP inference each assuming a hypertree MRF of increasing treewidth.  Since
MLNs of interest typically consists of at least one factor with three random variables, 
we will focus on triplewise MRFs as presented e.g. in~\cite{apsel14aaai} in order to investigate {\bf (Q1)}. Given an MRF $\mathcal{M}$ induced by an MLN 
over the set of (ground) random variables $x=\{x_i,\ldots,x_n\}$ and with factors $F = \{(\theta_f,x_f)\}_f$, the MAP-LP is define as follows, see also~\cite{apsel14aaai}.
 For each subset of indices $\mathcal{I}$ taken from $\{1,\ldots,n\}$ of size $1 \leq |\mathcal{I}| \leq 3$, let $\mu_\mathcal{I}$ denote a vector of variables of size $2^{|\mathcal{I}|}$. A notation $\mu_{ijk}(x_i,x_j,x_k)$ is 
used to describe a specific variable in vector $\mu_\mathcal{I}$ corresponding to the subset $\mathcal{I}=(i,j,k)$ and entry $(x_i,x_j,x_k) \in {\{0,1\}}^3$. Additionally, let $\mathcal{I}_F$ denote the set of all ordered indices for which there exists a factor $f$ with a matching variables scope $x_f$, and let $\theta_{ijk}$ denote the log probability table of a factor whose variables scope is $(x_i,x_j,x_k)$. The  MAP-LP is now defined as follows:
\begin{equation}\label{triple}
\begin{aligned}
&\text{maximize}_{\mu} \sum_{(i) \in \mathcal{I}_F} \langle \theta_{i}, \mu_{i} \rangle + \sum_{(i,j) \in \mathcal{I}_F} \langle \theta_{ij}, \mu_{ij} \rangle + \sum_{(i,j,k) \in \mathcal{I}_F} \langle \theta_{ijk}, \mu_{ijk} \rangle \\
&\begin{tabular}{ >{$}l<{$} >{$}l<{$} >{$}l<{$} >{$}l<{$}}
\text{subject to  } & \sum_{x_i} \mu_{i}(x_i) = 1 & \forall (i) \in \mathcal{I}_F, & \\ %\forall x_i \in \{0,1\}
								& \sum_{x_j} \mu_{ij}(x_i,x_j) = \mu_i(x_i) & \forall (i,j) \in \mathcal{I}_F, & \\ %\forall (x_i,x_j) \in {\{0,1\}}^2\\
								& \sum_{x_i} \mu_{ij}(x_i,x_j) = \mu_j(x_j) & \\
								& \mu_{ijk}(x_i,x_j,x_k) \geq 0 & \forall (i,j,k) \in \mathcal{I}_F, & \\ %\forall (x_i,x_j,x_k) \in {\{0,1\}}^3 \\
								& \sum_{x_k} \mu_{ijk}(x_i,x_j,x_k) = \mu_{ij}(x_i,x_j)& \\
								& \sum_{x_j} \mu_{ijk}(x_i,x_j,x_k) = \mu_{ik}(x_i,x_k)& \\
								& \sum_{x_i} \mu_{ijk}(x_i,x_j,x_k) = \mu_{jk}(x_j,x_k)&
\end{tabular} \ \\
&\text{where  } \langle \theta_{ijk}, \mu_{ijk} \rangle = \sum_{(x_i,x_j,x_k) }  \theta_{ijk}(x_i,x_j,x_k) \cdot \mu_{ijk}(x_i,x_j,x_k)
\end{aligned}
\end{equation}
This MAP-LP has to be instantiated for each MLN after we have computed a ground MRF induced by the MLN and a set of constants. 

The fact that both MLNs and RLPs are based on logic programming features a different and more convenient translation from an MLN to a 
MAP LP for inference. Actually, each MAP-LP is defined through weights, marginals, and triples of variables induced by the MLN formulas.
% in an MLN and every combination of values of atoms in that formula. 
This generation can naturally be specified at the lifted level as done in the RLP in Fig.~\ref{lp:map}. % shows the RLP doing exactly this. 
Since it is defined at the lifted level abstracting from the specific MLN, it clearly answers  {\bf (Q1)} affirmatively. To see this, consider
to perform inference in the well known smokers MLN. Here, there are two rules. The first rule says that smoking can cause cancer 
\begin{equation*}
 0.75 \ \mathtt{smokes(X) \ \imp \ cancer(X)}, 
 \end{equation*}
 and the second implies that if two people are friends then they are likely to have the same smoking habits 
 \begin{equation*}
 0.75 \ \mathtt{smokes(X) \ \imp \ friends(X, Y), smokes(X)}.
 \end{equation*}
Since the second formula contains three predicates using the triplewise MAP-LP is valid. Now the  
%The fact that both MLNs and RLPs are based on logic programming features an especially convenient translation from an MLN to a 
%MAP LP for inference. 
the MLN as well as the used constants are encoded in the following LogKB:
\begin{lstlisting}[language=ampl,frame=none,basicstyle=\footnotesize\ttfamily,numbers=none,backgroundcolor=\color{blue!10}]
person(anna).  person(bob).  ...   #the people in the social network
value(0).  value(1).               #we consider binary MLNs   

#encoding of the MLN clauses and their weights
w(smokes(X), cancer(X), 1, 0) = 0 :- person(X).
w(smokes(X), cancer(X), V1, V2) = 0.75 :- 
                   person(X), value(V1), value(V2).

w(friends(X, Y), smokes(X), smokes(Y), 1, 1, 1) = 0.75 :- 
                   person(X), person(Y).
w(friends(X, Y), smokes(X), smokes(Y), 1, 0, 0) = 0.75 :- 
                   person(X), person(Y).
w(friends(X, Y), smokes(X), smokes(Y), V1, V2, V3) = 0 :- 
                   person(X), person(Y), value(V1), value(V2), value(V3).

m(smokes(gary), 1).                     #the smokers we know  
m(smokes(helen), 1).  

m(friends(anna, bob), 1).               #the observed friendship relations
...
m(friends(helen,iris), 1).
\end{lstlisting} 
This examples gives us an opportunity to introduce another benefit of mixing logic with arithmetics, namely the ability to easily represent such concepts as indicator functions. In the MAP LP we can add an "indicator" predicate {\tt clause(X, Y)} which equals to 1 if {\tt X} and {\tt Y} are two atoms that form a clause in an MLN and is 0 otherwise. This results in the following extension to the LogKB:
\begin{lstlisting}[language=ampl,frame=none,basicstyle=\footnotesize\ttfamily,numbers=none,backgroundcolor=\color{blue!10}]
clause(smokes(X), cancer(X)) = 1 :- person(X).
\end{lstlisting} 
We can now write the LP constraints in a more mathematical way using constraints such as
\begin{lstlisting}[language=ampl,frame=none,basicstyle=\footnotesize\ttfamily]
subject to: {pred(P1), pred(P2), val(V2)} 
     (m(P2,V2) - clauseMarg2(P1,P2,V2))*clause(P1, P2) = 0;
\end{lstlisting}
The resulting program is equivalent to the previous version, but can be more straightforward for some people with an optimization background.
Both programs show that MAP-LP inference within MLNs can be compactly represented as an RLP. Only the LogKB changes when the
MLN and/or the evidence changes. Moreover, the RLP does not rely on the fact that we consider MLNs. Propositional 
models can be encoded in the LogKB, too. Hence the RLP extracts the essence of MAP-LP inference in probabilistic models, whether
relational or propositional. This supports even more that {\bf (Q1)} can be answered affirmatively.

\begin{figure*}[t]
\centering
\subfloat[Number of variables in the lifted and ground LPs.]{
\includegraphics[width=0.3\textwidth]{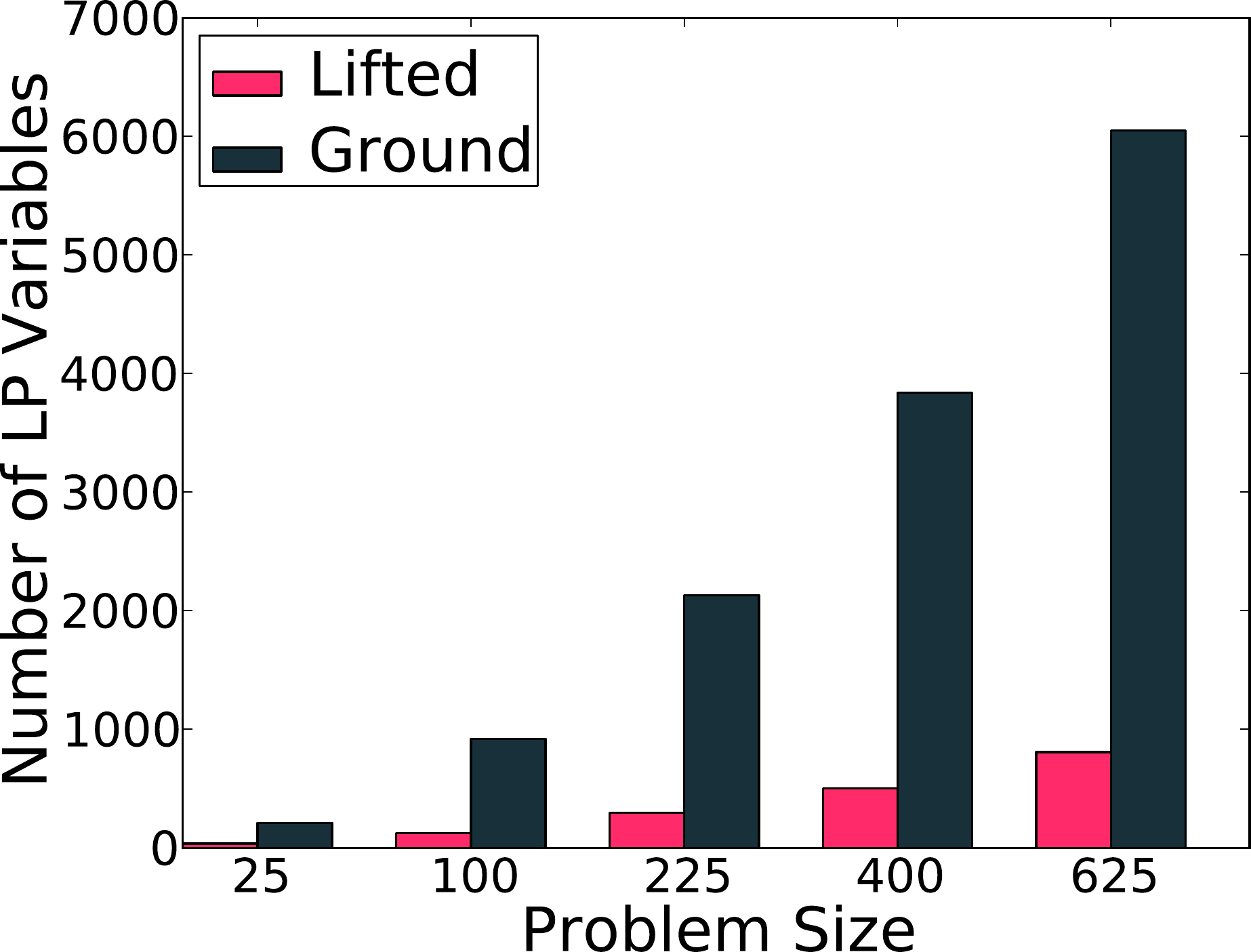}
}
\qquad
\subfloat[Time for solving the ground LP vs. time for lifting and solving.]{
\includegraphics[width=0.3\textwidth]{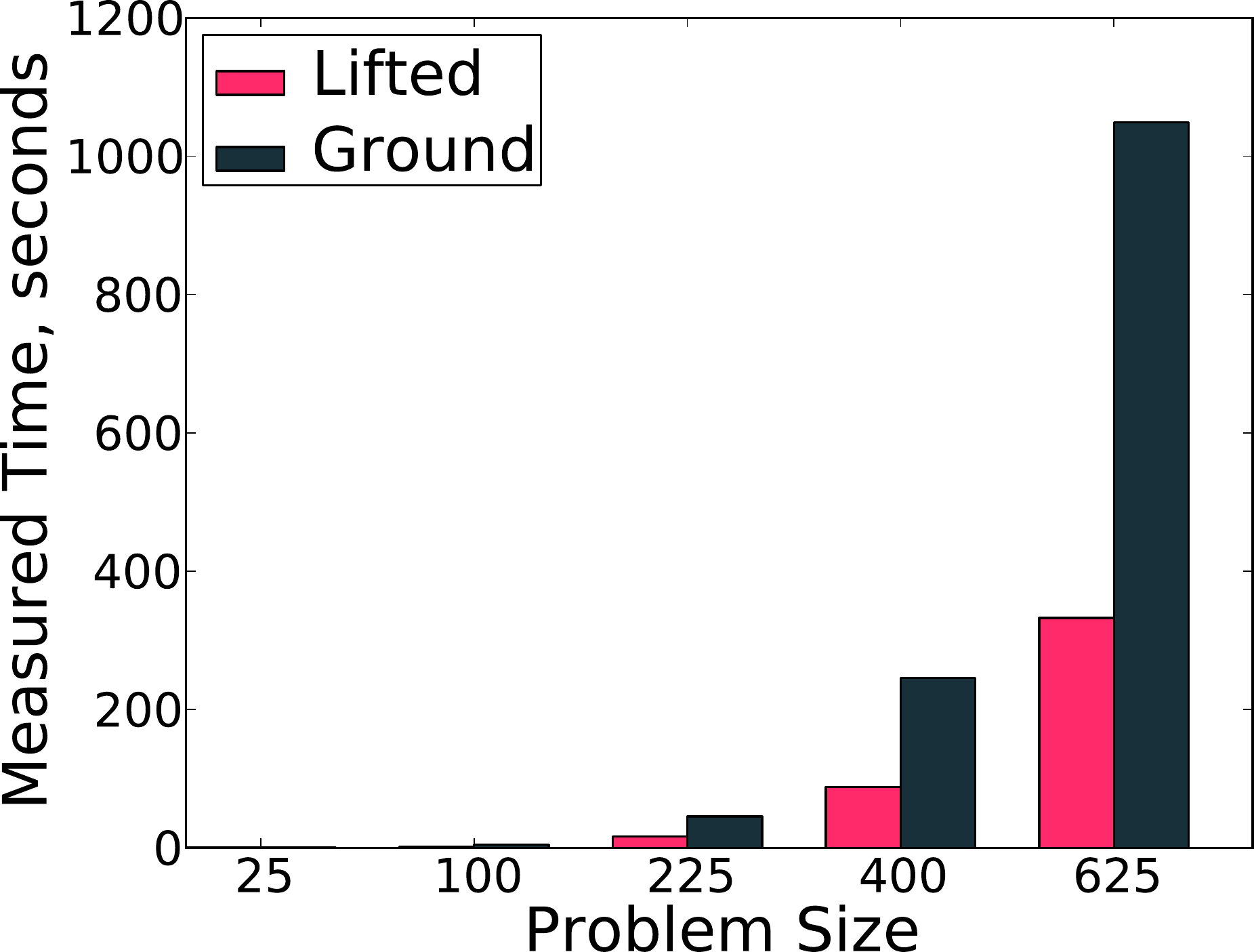}
}
\caption{\label{fig:ExpsMLN} Experimental results of relational linear programming for MAP-LP inference within Markov logic networks.}
\end{figure*}

Let us now turn towards investigating {\bf Q2}. As shown in previous works, inference in graphical models can be dramatically sped-up using lifted inference.
%Furthermore, a relaxed version of MAP inference can be solved using LPs, 
%following well-known
%relaxation, 
%see e.g.~\cite{globersonJ07} for details.
%\begin{align*}
%    \max\nolimits_{\mathbf{\mu}} \ \ \ \ & %\sum\nolimits_{ij\in E}\sum\nolimits_{x_i,x_j}\theta_{i,j}%(x_i,x_j)\mu_{ij}(x_i,x_j)\\
%    &\; \; \; + %\sum\nolimits_i\sum\nolimits_{x_i}\mu_i(x_i)\theta_i(x_i)\\ 
%    \text{s.t.} \ \ & \sum\nolimits_{x_i} \mu_{ij}(x_i,x_j) = \mu_j(x_j),\\ &\sum\nolimits_{x_j} \mu_{ij}(x_i,x_j) = \mu_j(x_i),\ \ \sum\nolimits_{x_i}\mu_i(x_i) = 1.
%\end{align*}
Thus, it is natural to expect that the symmetries in graphical models which can be exploited by standard lifted 
inference techniques will also be reflected in the corresponding MAP-(R)LP. To verify whether this is indeed
the case we induced MRFs of varying size from a pairwise smokers MLN~\cite{fierensKDCM12}. In turn, we
used a pairwise MAP-RLP following essentially the same structure as the triplewise MAP-RLP above but restricted to pairs. 
We scaled the number of random variables from $25$ to $625$ arranged
in a grid with pairwise and singleton factors with identical potentials. The results of the experiments can be seen in Figs.~\ref{fig:ExpsMLN}(a)
and (b). As Fig.~\ref{fig:ExpsMLN}(a) shows, the number of LP variables is significantly reduced.
Not only is the linear program reduced, but due to the fact that the lifting is carried out only once, we also measure a considerable
decrease in running time as depicted in Fig.~\ref{fig:ExpsMLN}(b). Note that the time for the lifted experiment includes the time needed to compile the LP.
This affirmatively answers {\bf (Q2)}.

\subsection{Programming Collective Classification using LP-SVM}
Networks have become ubiquitous, and often we are interested in how objects in these networks influence each other. 
Consequently, collective classification has received a lot of attention recently~\cite{chakrabartiDI98,neville2000iterative, neville2003collective, neville2007relational, richardson2006markov,senNBGGE08}. It refers to the task of jointly classifying a set of inter-related objects.
It exploits the fact that 
%although most of the conventional learning methods assume training instance independence in reality this assumption often doesn't hold. In many problems instances 
inter-related objects often share a lot of similarities. For example, in citation networks there are dependencies among the topics of a papers' references, 
and in social networks  people who are in a close contact tend to have similar interests. Using these dependencies instead of trying to fight them allows collective classification methods to outperform methods that assume object independence~\cite{chakrabartiDI98}.

Despite being successful, most of the research in collective classification so far has focused on generative models, at least as part of column generation approach
to solving quadratic program 
formulations of the collective classification task.
%Yet it is quite natural to phrase a collective classification task as a transductive inference problem \cite{vapnik1998statistical}. Transductive inference is a direct reasoning from observed to unobserved instances without an inductive hypothesis finding problem. That is, if we know a concept for a number of objects resp. instances, for example topics of papers, and have information about relationships between instances, for example citations between papers, we might try to predict topics for related but with unknown topics papers based on the relationship to the observed ones and, possibly, some other attributes. Although transductive SVMs are well-known,
%surprisingly little attention has been paid to use them for collective inference. Indeed, Klein {\it et al.}~\cite{kleinBS08} used generative models within
%a structured SVM. ~\cite{torkamaniL13}  
%but --- to the best of our knowledge --- no principled SVM approach has been proposed. 
%have been known since Vapnik, to the best of our knowledge they haven't been widely applied to collective classification.
%Yet it is known that, although generative models are more expressive, their asymptotic error is higher compared to the error of discriminative models \cite{ng2002discriminative}. 
%The main reasons for that is probably the lack of relational modeling tools allowing one to incorporate information about relations among objects into 
%mathematical programs. 
%discriminative models, which often serve as a basis for transductive models. 
We here illustrate that relational linear programming could provide a first step towards a principled large-margin approach. Specifically, we 
introduce a transductive\footnote{Transductive inference is a direct reasoning from observed to unobserved instances without an inductive hypothesis finding problem\cite{vapnik1998statistical}.} collective SVM based on RLPs. That is, since we observe a class label for a number of objects resp.~instances, say, the topics of papers and we have information about relationships among instances, say the citations among papers, we predict the topics of related but unlabeled papers based on the relationship to the observed ones and, possibly, some other attributes.

More precisely, we seek the largest separation between labeled and unlabeled
nodes through regularization within a relational linear program approximation to SVMs.  We start off by reviewing the vanilla LP approach to 
SVMs and then show how RLPs can be used to program a transductive collective classifier.
  \begin{figure}[t]
\begin{lstlisting}[language=ampl]
var slack/1;                              #the slacks
var weight/1;                             #the slope of the hyperplane 
var b/0;                                  #the intercept of the hyperplane
var r/0;                                  #margin

slacks=sum{label(I)} slack(I);                                #total slack
innerProd(I)=sum{attribute(_,J)} weight(J)*attribute(I,J);    #hyperplane

#find the largest margin. Here const encodes a trade-off parameter
minimize: -r + const * slacks;     

#examples should be on the correct side of the hyperplane 
subject to {label(I)}: label(I)*(innerProd(I) + b) + slack(I) >= r;
#weights are between 0 and 1
subject to {attribute(_, J)}: -1 <= weight(J) <= 1;
subject to : r >= 0;         #the margin is positive
subject to {label(I)}: slack(I) >= 0;        #slacks are positive
\end{lstlisting}
\caption{A linear programming SVM encoded as an RLP. Note, for convenience, we now use a {\tt minimize} instead of a {\tt maximize} statement.\label{rlp:lpsvm}}
\end{figure}

Support vector machines (SVMs)~\cite{vapnik1998statistical} are the most widely used model for discriminative classification at the moment. SVMs' hypothesis space is the space of linear models over numeric attributes of the data. Training is done by minimizing the number of misclassified examples and maximizing the gap between correctly classified examples and the separating hyperplane, with the squared norm of the weight vector as a normalization factor. This task is traditionally posed as a quadratic optimization problem (QP). Zhou {\it et. al.}~\cite{zhou2002linear} have shown that the same problem can be modeled as an LP with only a small loss in generalization performance. The LP they suggested is the following:
\begin{align*}
\operatorname{minimize}\quad  &\;-r + C \sum\nolimits_{i=1}^l \xi_i\\ 
\text{subject to } &\; y_i(\wv \xv_i + b) \geq r - \xi_i\;, \\
 &-1 \leq \wv_i \leq 1\;,\\
& \xi_i \geq 0 \;,\\
&r \geq 0\;,
\end{align*}
and we refer to \cite{zhou2002linear} for more details. This LP-SVM can readily be applied to classify papers in the Cora dataset~\cite{sen:aimag08}.
The Cora dataset consists of $2708$ scientific publications 
classified into one of seven classes. The citation network consists of 5429 links. Each publication in the dataset is described by a 0/1-valued word 
vector indicating the absence/presence of the corresponding word from the dictionary. The dictionary consists of $1433$ unique words. 
We turned this problem into a binary classification problem by taking the most common of 
$7$ classes as a positive class and merging the other $6$ 
into a negative class.  

To do so, we have to provide the RLP and the corresponding LogKB.  The RLP in Fig.~\ref{rlp:lpsvm} encodes the 
the vanialla LP-SVM. Since it is not collective, we completely ignored the 
citation information and classified documents based on 0/1 word features using the following LogKB: 
%A training set at hand is represented by a LogKB such as:
\begin{lstlisting}[language=ampl,frame=none,basicstyle=\footnotesize\ttfamily,numbers=none,backgroundcolor=\color{blue!10}]
const = 0.021.               
               
attribute(31336, 119).  attribute(31336, 126).  ...
label(17798) = -1.  label(10531) = 1.  ...
\end{lstlisting}
Here the arguments of the {\tt attribute/2} predicate are indices of a document and a word present in a document respectively. Words that are not present in a document, i.e., they have value zero in the original dataset, are not specified. Again, this answers {\bf (Q1)} affirmatively, since the RLP stays fix for different datasets.

We now show how to transform this vanilla (R)LP-SVM model into a transductive collective one (TC-RLP-SVM) just by programming. 

Indeed, there are
a number of ways to do this. We chose the following approach. We added constraints which ensure that unlabeled instances have the same label as the labeled ones to which they are connected by a citation relation. To account for contradicting examples, in the spirit of SVMs, we introduced slack variables for these constraints and added them to the 
objective with a separate parameter. This resulted in the changes to the RLP-SVM shown in Fig.~\ref{rlp:tclpsvm}.
Here, the new predicate {\tt pred/2} denotes the predicted label for unlabeled instances. The LogKB gets two new predicates:
\begin{lstlisting}[language=ampl,frame=none,basicstyle=\footnotesize\ttfamily,numbers=none,backgroundcolor=\color{blue!10}]
const(1) = 0.0021.  const(2) = 0.0031.

cite(89547, 1132385).  cite(89547, 1152379).  ...
query(1128959).  query(16008).  ...
\end{lstlisting}
\begin{figure}[t]
\begin{lstlisting}[language=ampl]
var pred/1;      #predicted label for unlabeled instances
var slack/2;     #slack between neighboring instances
... 
slacks1 = sum{label(I)} slack(I);                       #total label slack
slacks2 = sum{label(I1,l2)} slack(I1,l1);         #total inter-label slack
#find the largest margin. Here the consts encode trade-off parameters
minimize: -r + const(1) * slacks1 + const(2) * slacks2;              
...   
#examples should be on the correct side of the hyperplane            
subject to {query(I)}: pred(I) = innerProd(I) + b;
#related instances should have the same labels. 
subject to {cite(I1, I2), label(I1), query(I2)}: 
	label(I1) * pred(I2) +  slack(I1, I2) >= r;
#the symmetric case	
subject to {cite(I1, I2), label(I2), query(I1)}: 
	label(I2) * pred(I1) + slack(I1, I2) >= r;
\end{lstlisting}
\caption{An RLP-SVM model for collective inference in a transductive setting. Shown are only the changes and add-ons to the vanilla RLP-SVM model from Fig.~\ref{rlp:lpsvm}.\label{rlp:tclpsvm}}
\end{figure}
The {\tt cite/2} predicate represent citation information, and the {\tt query} predicate marks unlabelled instances whose labels are to be inferred. We notice that the parameters in the objective play a radically different role in the TC-RLP-SVM. In the vanilla case a parameter has to be carefully chosen on the training phase, but then prediction is done using the learned weight vector only. In the transductive setting the linear model, which is in the heart of RLP-SVM, plays a role of a media between labeled and unlabeled instances. The weights are tuned for every new problem instance (remember, that the whole point of transductive inference is not to learn an intermediate predictive model, but to do inference directly). In this case the objective parameters become a lot more important. In turns out that the optimal value of the parameters depends on the size of a problem and hence they have to be tuned during the transductive inference phase as well using e.g. cross-validation.
 \begin{figure}[t]
 \begin{center}
\subfloat[Box plot for prediction errors of collective and vanilla SVMs using RLPs.\label{fig:box}]{
 \includegraphics[width=0.45\textwidth]{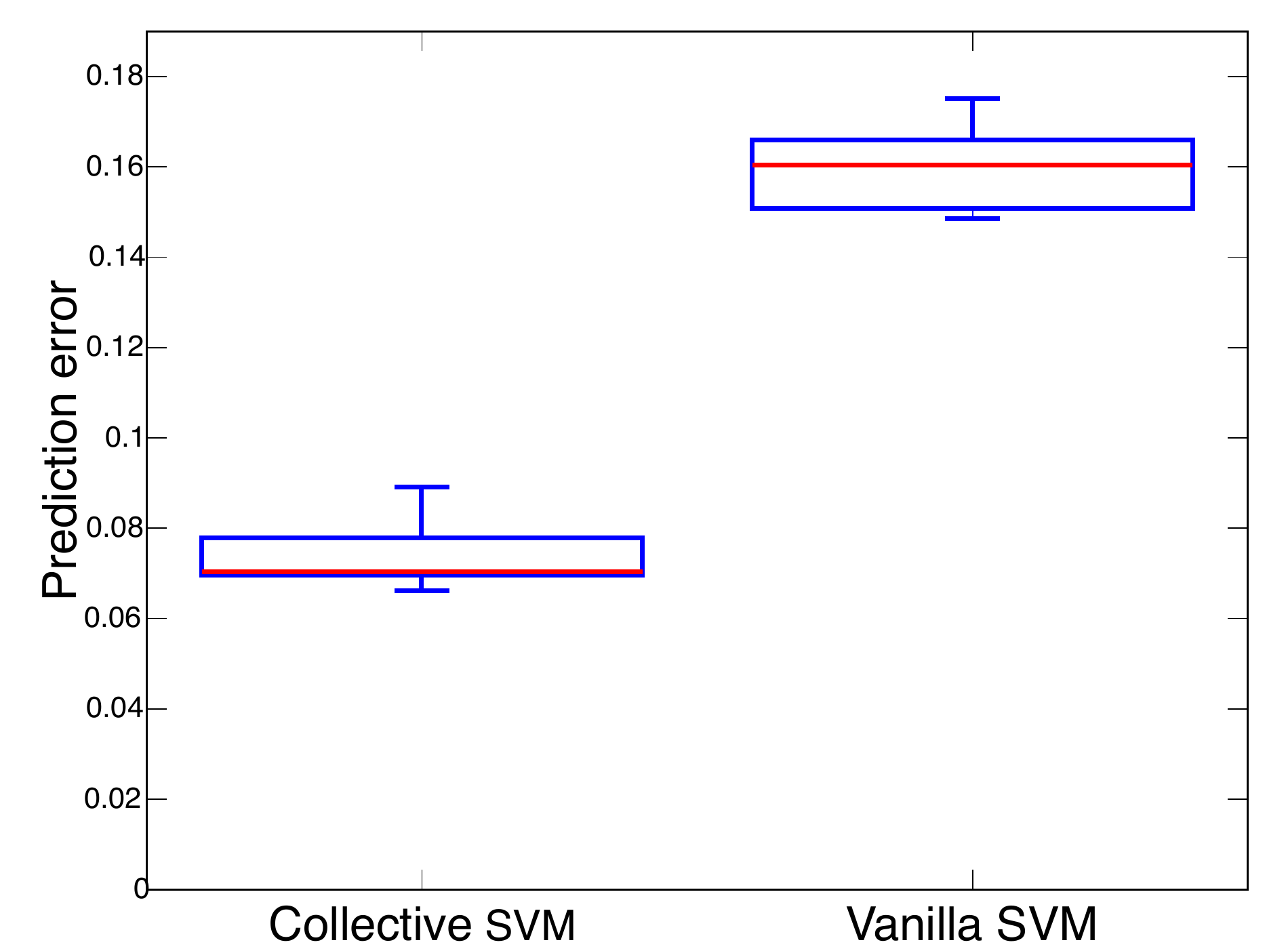}
}\hfill
\subfloat[Lifted solving RLP-SVMs.\label{fig:liftedsvm}]{
\begin{minipage}[b][1\width]{% 
0.5\textwidth} 
\centering
 \includegraphics[width=0.85\textwidth]{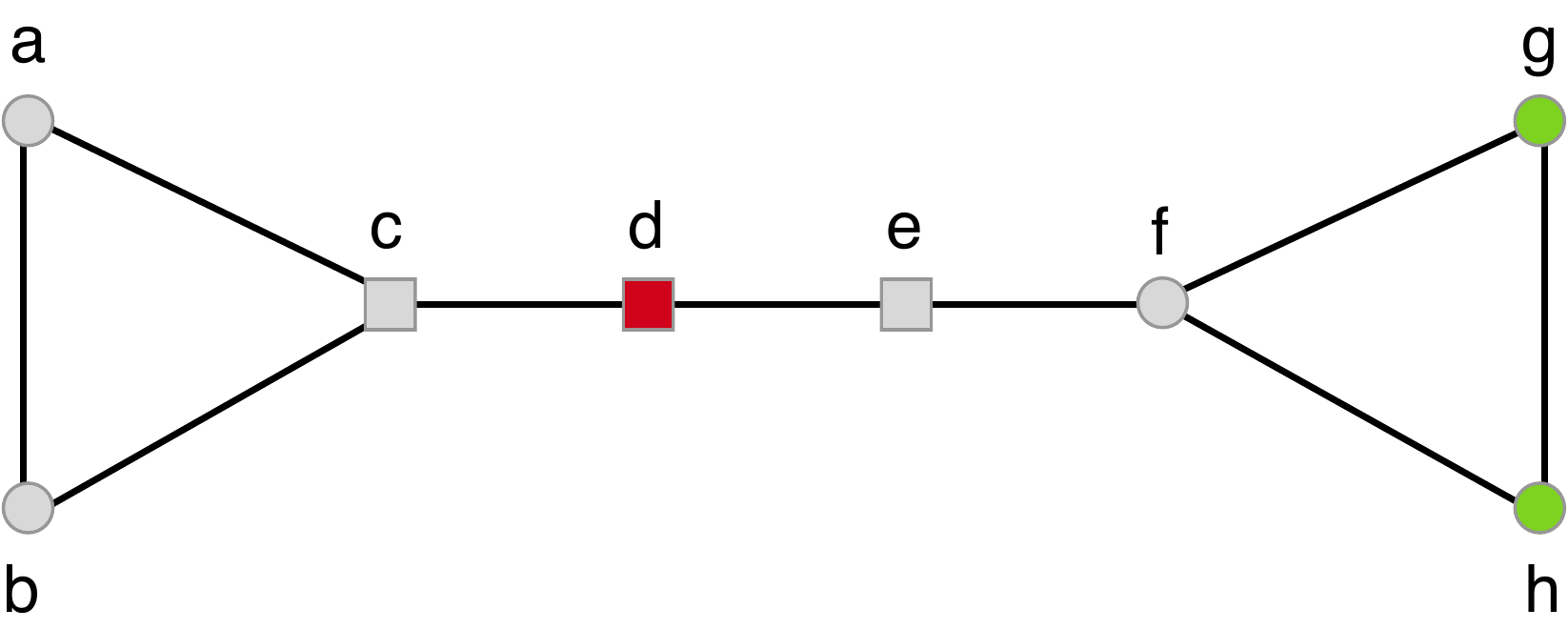}
 \includegraphics[width=0.85\textwidth]{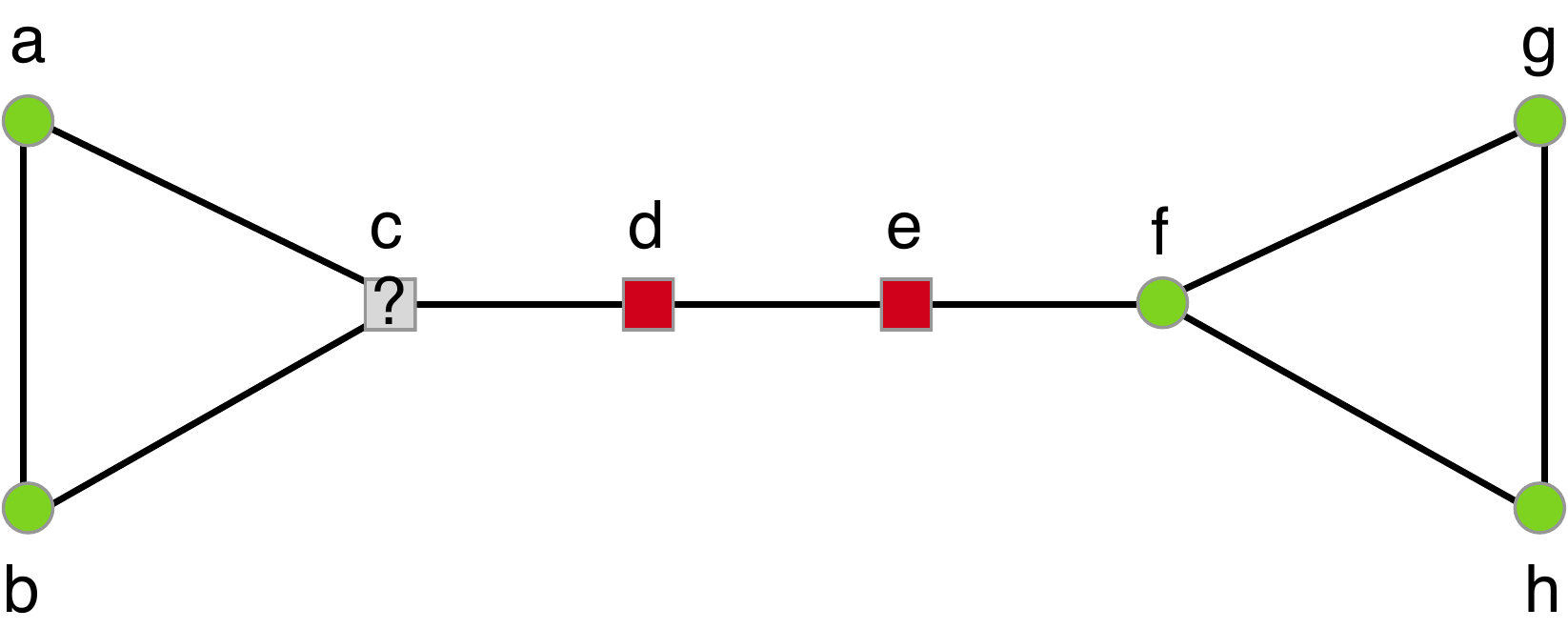}
 \includegraphics[width=0.85\textwidth]{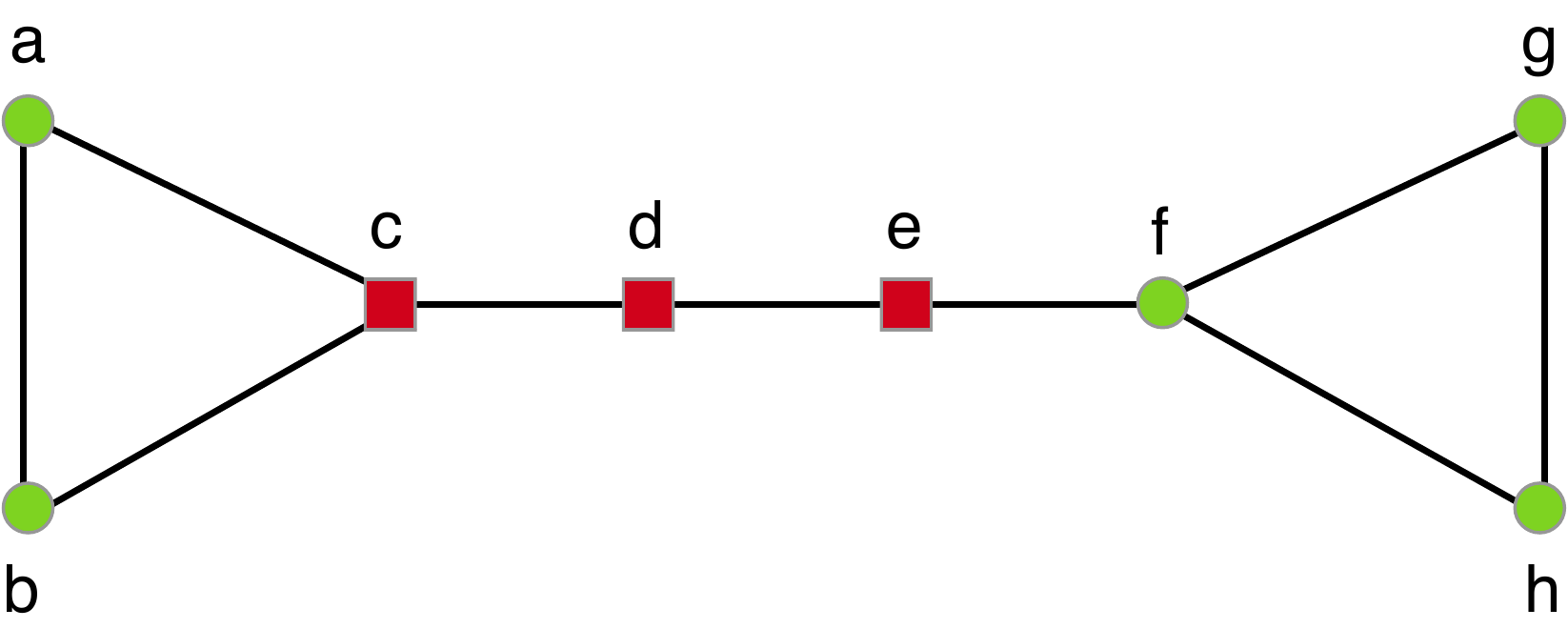}
\end{minipage}} 
 \caption{Experimental results of linear programming support vector machines. (a) TC-RLP-SVM (left) versus vanilla RLP-SVM on the CORA dataset. (b) Liifted RLP-SVM experiments. (Top) Training set. Nodes can have different shapes and different degrees. The task is to predict the colors red resp.~green for the gray nodes. (Middle) predictions of the vanilla RLP. The `?' indicates that the SVM cannot make a definite prediction. (Bottom) Predictions of the TC-RLP-SVM. In both cases the induced LPs were solved using lifted linear programming}
 \end{center}
 \end{figure}

To illustrate the benefit of the TC-RLP-SVM we compared its performance with that of the vanilla RLP-SVM on the task of paper topic classification in the Cora dataset. The experiment protocol was as follows. We first randomly split the dataset into a training set $A$, a validation set $C$, and a test set $B$ 
in proportion $70/15/15$. The validation set was
used to select the parameters of the TC-RLP-SVM in a $5$-fold cross-validation fashion. That is, we
split the validation set into $5$ subsets $C_i$ of equal size. On these sets we performed parameter selection by doing a grid search for each $C_i$ on a 
$A \cup (C  \setminus C_i)$ labeled and $B\cup C_i$ unlabeled examples, computing the prediction error on $C_i$ and averaging it over all $C_i$s. We then evaluate the selected parameters on the test set $B$ whose labels were never revealed in training. We repeat this experiment $5$ times, one for each $C_i$,
 for both the TC-RLP-SVM and the vanilla RLP-SVM. 
 The results are summarized in Fig.~\ref{fig:box}. The vanilla RLP-SVM achieved a prediction error of  $16\pm1\%$. The TC-RLP-SVM 
 achieved $7.5\pm1\%$. A paired t-test ($p=0.05$) revealed that the difference in mean is significant. Although best performance was not
 our goal, the performance is quite encouraging. For instance, Klein {\it et al.}~\cite{kleinBS08} reported
 error rates of about $13\%$ for there structured SVM although using different folds. In any case, 
 %%with a confidence level of 95\%: $t(4) = 13.1$, $p = 0.0002$, i.e. the difference is statistically significant. 
 just reprogramming the RLP-SVM --- adding three relational constraints --- resulted in a $50\%$ error reduction. This clearly answers {\bf (Q3)}
 affirmatively. %: just by programming we get a significant boost in classification performance. 
 %rules we achieved a clear improvement by simply adding 3 constraints to the standard SVM.

 \subsection{Lifted Solving Relational Linear Programming Support Vector Machines}
Finally, to close the loop, we illustrate that RLP-SVMs are liftable, too. Although not surprising from a theoretical perspective --- we have already shown
that any LP that contains symmetries is liftable --- this constitutes the very first symmetry-aware SVM solver and is an encouraging sign that lifting goes beyond probabilistic inference, i.e.,  lifted statistical machine learning may not be insurmountable.

The problem we considered is the one of predicting the color (red/green) of a vertex in an extended version of the so called McKay graph as shown in Fig.~\ref{fig:liftedsvm}(top). 
The only attributes used for classification are the degree of a vertex and its shape (circle or square). Gray nodes are unlabeled. 
This resulted in the following LogKB for
the TC-RLP-SVM:
\begin{lstlisting}[language=ampl,frame=none,basicstyle=\footnotesize\ttfamily,numbers=none,backgroundcolor=\color{blue!10}]
edge(a, b). ... edge(h, g).

label(f) = 1.  label(g) = 1.  label(d) = -1.

sim_edge(X, Y) :- edge(X, Y).
sim_edge(X, Y) :- edge(Y, X).

attribute(a, shape) = 1. ...  attribute(h, shape) = 1.
\end{lstlisting}
and the following additional definition in the RLP:
\begin{lstlisting}[language=ampl,frame=none,basicstyle=\footnotesize\ttfamily,numbers=none]
attribute(X, degree) :- sum <sim_edge(X, _)>  1;
\end{lstlisting}
Due to the model/instance separation property of RLPs we can directly apply the vanilla RLP-SVM from Fig.~\ref{rlp:lpsvm}. This resulted in the color 
predictions shown in Fig.~\ref{fig:liftedsvm}(middle). As one can clearly see, the vanilla LP-SVM can predict correctly colors for nodes {\tt a, b ,e} and {\tt f}, but fails 
to predict a color for the node {\tt c} ( it gives 0 prediction).

Then, we added the same collective constraints we used already in the RLP shown in Fig.~\ref{rlp:tclpsvm}. That is, we constraint uncolored nodes to have the same color as their colored neighbors where possible). Using this neighborhood information allowed the collective LP-SVM to correctly classify all the nodes in the graph as show in 
Fig.~\ref{fig:liftedsvm}(bottom). The reason seems to be that the vanilla RLP-SVM uses both degree and color attributes whereas the collective RLP-SVM
can set the degree weight to zero due to the additional constraints and, hence, makes predictions based only on the shape attribute, which 
is enough to achieve a perfect classification on this graph. 

More interestingly, in both cases there is lifting, i.e., the dimensions of the (induced) LP-SVMs 
were reduced. More precisely, the RLP-SVM
was reduced from $46$ variables and constraints down to $22$, only $48\%$ of the original size. 
The collective version was compressed even slightly more, namely from $63$ down to $29$. This is $46\%$ of the original size. Most interestingly,
the lifted collective RLP-SVM is of smaller size than the original vanilla LP-SVM. This supports an affirmative answers of {\bf (Q2)}.\\

Taking all results together, the illustrations clearly show that all three questions {\bf (Q1)} -- {\bf (Q3)} can be answered affirmatively.

% All the previous examples serve to illustrate flexibility, expressive power and compactness of the language. One can express generative and declarative models equally easy. 
%%Representing relational concepts such as passive smoking, or shortest path in a graph is as simple as writing a logical rule. 
%Logical concepts can be naturally mixed into arithmetic expressions. The language is quite minimalistic by itself and using typed variables in a LogKB allows to avoid redundancy in problem instance definitions. These properties not only simplify specification of existing models, but also facilitate the development of novel ones, like in the example with collective classification. 

%!TEX root = main.tex

\section{Future work}
\label{fut}

Relational linear programming is attractive for many AI and machine learning problems, but much remains to be done. For instance,
the language could be extended with the concepts of modules and name spaces, allowing one to build libraries of relational programs,
as well as combining it with Mattingley and Boyd's~\cite{MattingleySB12} CVXGEN to automatically generate 
custom C code that compiles into a reliable, high speed solver for the problem family at hand.
The framework should also be extended to other mathematical programs such as integer LPs, mixed integer programs, 
quadratic programs, and semi-definite programs, among others. Together with the declarative
nature of this relational mathematical programming approach to AI, one should investigate 
program analysis approaches to automate problem decomposition at a lifted level. If symmetries could be detected and exploited
efficiently in other mathematical program families, too, this would put general symmetry-aware machine learning and AI 
even more into reach. It would also be interesting to investigate infinite relational linear programs. 

%In the current implementation the LogKB supports a very restricted subset of the first-order logic. This allows us to apply efficient bottom-up grounding techniques, but also restricts the expressive power of the language. In principle, one can simply use some efficient Prolog implementation for grounding and by that make the language Turing-complete. We illustrate the possible implications of this extension on the graph kernels design problem.

The most attractive immediate avenue, however, is to explore relational linear programming within AI and machine learning tasks. 
First of all, the novel collective classification approach should be rigorously be evaluated and compared to other approaches 
on a number of other benchmark datasets. Other attractive avenues are the exploration of the symmetry-aware SVMs outlined in the present paper
within other learning setting, 
relational dimensionality reduction via LP-SVMs~\cite{BiBEBS03}, novel relational boosting approaches
via linear programs~\cite{DemirizBS02}, and developing relational and lifted solvers for computing optimal Stackelberg strategies in 
two-player normal-form game~\cite{conitzerS06}, among others. One should also push the programming view on relational 
machine learning tasks.  As a prominent example consider kernels for classifying graphs. Graphs classification is a very important task in bioinformatics~\cite{airola2008all}, 
natural language processing \cite{suzuki2003hierarchical} and many other fields. Kernelised SVM is often the method of choice. The idea is to 
represent an SVM  optimization objective in such a way that instance vectors only appear in inner products with each other. These inner products 
can then be replaced by functions (kernels) that effectively represent inner products in higher dimensional space. This inner product view
on one hand makes SVM a non-linear classifier but also allows one to deal with structured objects such as graphs e.g. using convolution
kernels~\cite{haussler1999convolution}. Convolutions kernels introduced the idea that kernels could be built
to work with discrete data structures interactively from kernels for smaller composite parts. RLPs suggests to view them 
as a programming task. Within the logical knowledge base we define the parts and a generalized sum over products --- a generalized
convolution --- is realized within the relational mathematical program. Similarly, many other graph kernels known could be realized.
%Selecting a kernel function that is a good approximation of similarity between two graphs (which is in fact the problem of determining graph isomorphism) and is at the same time tractable is the main focus of the graph kernels design. Generally speaking, graph kernels are based on comparison of graph substructures. 
Walks \cite{gartner2002exponential}, cyclic patterns \cite{horvath2004cyclic} and shortest-paths \cite{borgwardt2005shortest} are examples of substructures considered so far. One can easily see that all these concepts are naturally representable as a logic programs in Prolog. Assume e.g. 
that $\mathtt{shortestPath(A, B, G, Paths)}$ computes in {\tt Path} the shortest path between nodes {\tt A} and {\tt B} in graph {\tt G}
%In fact, one can imagine formulating a knowledge base describing the most common graph substructures. I.e. the concept of all the shortest paths in a graph can be described by the following Prolog rules:
For instance, we can program the shortest path in Prolog as follows:
%\begin{minted}[linenos,
%	      fontsize = \footnotesize,
%               numbersep=5pt]{rlp}
%path(A, B, G, Len) :- travel(A, B, G, [A], Len).
%
%travel(A, B, G, P, L) :- connected(A, B, G, L).
%travel(A, B, G, Visited, L) :- connected(A,C,D), C \== B,
%                               \+member(C,Visited), 
%                               travel(C, B, [C|Visited], L1), 
%                               L is D + L1.
%
%shortest(A, G, Paths) :- length(_, G, B), setof([A, B, Length], 
%                          shortest(A, B, G, Length), Paths).
%
%shortest(A, B, G, Length) :- setof(L,path(A,B,L),Set), 
%                             Set = [_|_], minimal(Set, Length).
%\end{minted}
We can then use it to define the convolution kernel in an RLP SVM as follows:
\begin{lstlisting}[language=ampl,frame=none,basicstyle=\footnotesize\ttfamily]
k(G1, G2) = sum{vertex(G1,V1), vertex(G1,V2), 
  shortestPath(V1,V2,G1,L1),  vertex(G2,V3), vertex(G2,V4), 
  shortestPath(V3,V4,G2,L2)}       
             simple_k(V1,V2,L1,V3,V4,L2).
simple_k(V1,V2,L1,V3,V4,L2) = ...  
\end{lstlisting}
where {\tt simple\_k} is any kernel on vertices and lengths. %That is relational mathematical programming suggests a programming view on graph kernels. 
%For instance, the shortest paths can directly be used within an Designing and implementing graph kernels then boils down to simply combining these predefined concepts in an aggregate expression corresponding to a kernel in an LP SVM template. 
That is, instead of introducing a small change as a novel kernel and proving that it is a valid kernel, one just programs it; one separates the kernel programming from the mathematical program.

%!TEX root = main.tex

\section{Conclusion}
We have introduced relational linear programming, a simple framework
combining linear and logic programming. 
Its main building block are relational linear programs (RLPs). They are compact LP templates defining
the objective and the constraints through the logical concepts of individuals, relations, and quantified variables. 
This contrasts with mainstream LP template languages such AMPL, which mixes imperative and linear programming, and 
allows a more intuitive representation of optimization problems over relational domains where we have
to reason about a varying number of objects and relations among them, without enumerating them.
Inference in RLPs is performed by lifted linear programming. That is, symmetries within the ground linear program are employed to reduce its dimensionality, if possible, and the reduced program is solved using any off-the-shelf linear program solver. This significantly
extends the scope of lifted inference since it paves the way for lifted LP solvers for linear
assignment, allocation and and any other AI task that can be solved using LPs. 
Empirical results on approximate inference in Markov logic networks using LP relaxations, on solving Markov decision processes, and on collective inference using LP support vector machines illustrated the promise of relational linear programming. 

\section*{Acknowledgements} This research was partly supported by the Fraunhofer ATTRACT fellowship STREAM, by the EC under contract number FP-248258-First-MM, and by the German Science Foundation (DFG), KE 1686/2-1.




\bibliographystyle{alpha} 
\bibliography{biblio}




%\end{thebibliography}`'
\end{document}